\documentclass[conference,compsoc]{IEEEtran}
\usepackage[nocompress]{cite}
\usepackage{microtype}
\usepackage{graphicx}
\usepackage{booktabs}

\usepackage{amsmath}
\usepackage{amssymb}
\usepackage{mathtools}
\usepackage{amsthm}
\usepackage{algorithm}
\usepackage{algorithmic}
\usepackage[table]{xcolor}
\usepackage{xspace}
\usepackage{pifont}
\usepackage{threeparttable}
\usepackage{multirow}
\usepackage[hidelinks]{hyperref}
\usepackage[capitalize,noabbrev]{cleveref}

\theoremstyle{plain}
\newtheorem{theorem}{Theorem}[section]
\newtheorem{proposition}[theorem]{Proposition}

\newtheorem{corollary}[theorem]{Corollary}
\theoremstyle{definition}

\newtheorem{assumption}[theorem]{Assumption}
\theoremstyle{remark}
\newtheorem{remark}[theorem]{Remark}

\newcommand{\Keff}{K_{\text{eff}}}

\newcommand{\ie}{i.e.,\xspace}

\newcommand{\cmark}{\ding{51}}
\newcommand{\xmark}{\ding{55}}
\newcommand{\pmark}{{\color{orange}$\boldsymbol{\sim}$}}

\definecolor{randcol}{RGB}{214,183,165}
\definecolor{kindonlycol}{RGB}{190,143,112}
\definecolor{kindsemcol}{RGB}{123,125,103}
\definecolor{fullfpcol}{RGB}{165,162,132}
\definecolor{theorycol}{RGB}{205,203,188}

\newenvironment{packeditemize}{
	\begin{list}{$\bullet$}{
			\setlength{\labelwidth}{4pt}
			\setlength{\itemsep}{0pt}
			\setlength{\leftmargin}{\labelwidth}
			\addtolength{\leftmargin}{\labelsep}
			\setlength{\parindent}{0pt}
			\setlength{\listparindent}{\parindent}
			\setlength{\parsep}{0pt}
			\setlength{\topsep}{1pt}}}{\end{list}}

\DeclareMathOperator*{\argmin}{arg\,min}
\DeclareMathOperator{\tr}{tr}

\DeclareMathOperator{\Corr}{Corr}
\DeclareMathOperator{\Var}{Var}
\DeclareMathOperator{\Cov}{Cov}
\DeclareMathOperator{\MSE}{MSE}

\usepackage{tcolorbox}
\usepackage{fontawesome6}
\usepackage{simpleicons}

\definecolor{ghdark}{HTML}{24292F}
\definecolor{hfyellow}{HTML}{FFD21E}
\definecolor{hfink}{HTML}{1F2937}
\newcommand{\linkpill}[5]{%
  \href{#5}{\tcbox[on line,colback=#1,colframe=#1,boxrule=0pt,arc=7pt,
    left=6pt,right=7pt,top=2.5pt,bottom=2.5pt,boxsep=0pt,
    fontupper=\sffamily\upshape\small\color{#2}]{#3\hspace{0.45em}#4}}}
\usepackage{eso-pic}
\newcommand{\venuefooter}[1]{\AddToShipoutPictureFG*{\AtTextLowerLeft{\raisebox{-0.3in}{\footnotesize #1}}}}
\input{ot1ptm.fd}\input{ts1ptm.fd}
\DeclareFontShape{TS1}{ptm}{m}{sc}{<->ssub * ptm/m/n}{}
\DeclareFontShape{OT1}{ptm}{m}{scit}{<->ssub * ptm/m/it}{}
\let\appendix\appendices
\newenvironment{ack}{\section*{Acknowledgments and Disclosure of Funding}}{}
\hypersetup{pdftitle={Diversity Combining for Multi-Path LLM Reasoning},
pdfauthor={Guangsheng Yu, Litianyi Zhang, Qin Wang, Xu Wang, Mingyuan Li, Shaoxiong Ji, Ren Ping Liu, Massimo Piccardi}}

\begin{document}
\title{Diversity Combining for Multi-Path LLM Reasoning}
\author{\IEEEauthorblockN{Guangsheng Yu\textsuperscript{1}, Litianyi Zhang\textsuperscript{2}, Qin Wang\textsuperscript{3}, Xu Wang\textsuperscript{1},\\
Mingyuan Li\textsuperscript{4,5}, Shaoxiong Ji\textsuperscript{4,5}, Ren Ping Liu\textsuperscript{1}, and Massimo Piccardi\textsuperscript{1}}
\IEEEauthorblockA{\textsuperscript{1}University of Technology Sydney, \textsuperscript{2}The University of Sydney, \textsuperscript{3}CSIRO\\
\textsuperscript{4}ELLIS Institute Finland, \textsuperscript{5}University of Turku\\[0.7em]
\linkpill{ghdark}{white}{\faGithub}{OniReimu/DiversityCombining}{https://github.com/OniReimu/DiversityCombining}\hspace{0.7em}%
\linkpill{hfyellow}{hfink}{\simpleicon{huggingface}}{UTSCybeR/DiversityCombining}{https://huggingface.co/datasets/UTSCybeR/DiversityCombining}}}

\maketitle
\venuefooter{Accepted by NeurIPS 2026}

\begin{abstract}
Multi-path reasoning methods such as self-consistency (SC) sample $K$ reasoning paths and choose the most frequent answer. However, their gains quickly plateau as $K$ increases, and locating this saturation point currently requires fitting a scaling model to accuracies measured at several values of $K$ or stopping online per query.
We formalize multi-path LLM reasoning as a diversity combining problem from wireless communications: each path is a noisy channel observation, and the pairwise correlation of path correctness caps the design-effect effective sample size of the vote at a finite ceiling.
Generalized least squares (GLS) analysis shows that, for paths with equal gains, the optimal linear combiner of latent embeddings is uniform exactly when every row of the path covariance has the same sum, as under exchangeability, supporting majority vote as the natural default in standard SC while leaving room for weighting or pruning under heterogeneous prompt-template branches.
Across 5 models and 12 benchmarks, prompt-template diversity reduces path correlation in $55$ of $60$ model-benchmark cells, with a mean relative reduction of $16\%$ and reductions in every task domain. We derive an Adaptive-K rule that uses a four-path pilot to select $K^*$, retaining $96$--$103\%$ of MV@$K{=}32$ accuracy across Math, QA, and NLU.
\end{abstract}

\section{Introduction}
\label{sec:intro}

Multi-path reasoning has become a standard technique for improving LLM answer reliability.
Self-consistency~\cite{wang2023selfconsistency} samples $K$ chain-of-thought paths and returns the majority-vote answer; extensions include weighted voting~\cite{cisc2025}, tree-structured search~\cite{yao2023tot}, and token-level embedding aggregation~\cite{softcot2025}.

However, a common empirical observation lacks a satisfactory theoretical explanation: accuracy gains from additional reasoning paths diminish rapidly, and beyond a model-dependent threshold $K^*$, adding more paths yields negligible improvement.
Existing analyses attribute this saturation to answer-space coverage or sampling temperature, yet these explanations remain qualitative. The quantitative predictor of compound-inference scaling~\cite{chen2024morellmcalls} fits a query-difficulty model to accuracies measured at a few small values of $K$ and extrapolates to larger budgets, and online stopping rules~\cite{aggarwal2023adaptive,li2024esc} decide per query without a model-level $K^*$.
They also do not establish a formal relationship between the observable path agreement rate and the underlying correlation structure, which is what a saturation formula requires.

The classical Condorcet jury theorem~\cite{condorcet1785} predicts that majority voting among independent voters, each correct with probability above $1/2$, converges to certainty as $K\to\infty$, a guarantee that dependence among voters can break~\cite{berg1993condorcet}. Compound LLM inference indeed saturates much earlier~\cite{chen2024morellmcalls}, indicating that path correctness is correlated across questions. While ensemble diversity measures~\cite{kuncheva2003,jeffares2023ensemble} quantify disagreement, they do not provide a closed-form saturation bound for LLM reasoning. We address this gap via an analogy to \textit{diversity combining} over \textit{multipath channels}: a receiver observes $K$ noisy signal copies, and correlation among the copies reduces the combining gain they deliver~\cite{tse2005fundamentals}, so that the effective rank of the correlation matrix, rather than $K$, counts the copies that act as independent observations. We show that the correctness of reasoning paths sampled from the same model and prompt is positively correlated across questions, and that this correlation limits the gain from increasing $K$.
Fig.~\ref{fig:framework} illustrates the framework. The main contributions are:
\begin{packeditemize}
  \item \textbf{LLM-specific saturation diagnostic.} Adapting the classical design-effect~\cite{kish1965} and participation-ratio formulations to multi-path LLM reasoning, we cast reasoning paths as correlated branches and obtain the vote-level diagnostic $\Keff^{\text{vote}} = K/(1+(K{-}1)c)$ with ceiling $1/c$, where $c$ is the pairwise correctness correlation, a vote-level overdispersion measurable from outputs alone. Across three model families on GSM8K, $\Keff^{\text{vote}}$ at $K{=}32$ reaches $96$--$98\%$ of this ceiling (\S\ref{sec:theory}, \S\ref{sec:experiments}).
  \item \textbf{Uniform-weighting condition.} Applying classical GLS analysis to the latent-embedding combiner, we show that for paths with equal gains the optimal linear weights are uniform exactly when every row of the path covariance has the same sum, a condition the equicorrelated model satisfies, supporting majority vote as the natural default in standard SC and leaving room for weighting or pruning when prompt-template branches become heterogeneous (\S\ref{sec:theory}, \S\ref{app:hetero_slot}).
  \item \textbf{Broad decorrelation.} Across 5 models and 12 benchmarks, prompt-template perturbation reduces path correlation in $55$ of $60$ model-benchmark cells (mean $\Delta\rho{=}{-}16\%$, median ${-}15\%$). The reduction appears in all six task domains, and its size varies with the model-task cell rather than with the openness of the answer space (\S\ref{sec:cross_bench}, Fig.~\ref{fig:drho_cells}).
  \item \textbf{Adaptive-K design rule.} The diagnostic yields a closed-form operating point $K^*$ (Eq.~\ref{eq:adaptive_k}) from a fixed $K{=}4$ pilot, with no held-out tuning and no per-instance scorer. Cross-domain validation (Math, QA, NLU) shows this rule retains $96$--$103\%$ of accuracy at $K^*$ (\S\ref{sec:adaptive_k}); positioning relative to online-stopping and scaling-law approaches is summarized in Table~\ref{tab:comparison}.
\end{packeditemize}

\begin{figure*}[t]
\centering
\includegraphics[width=0.9\textwidth]{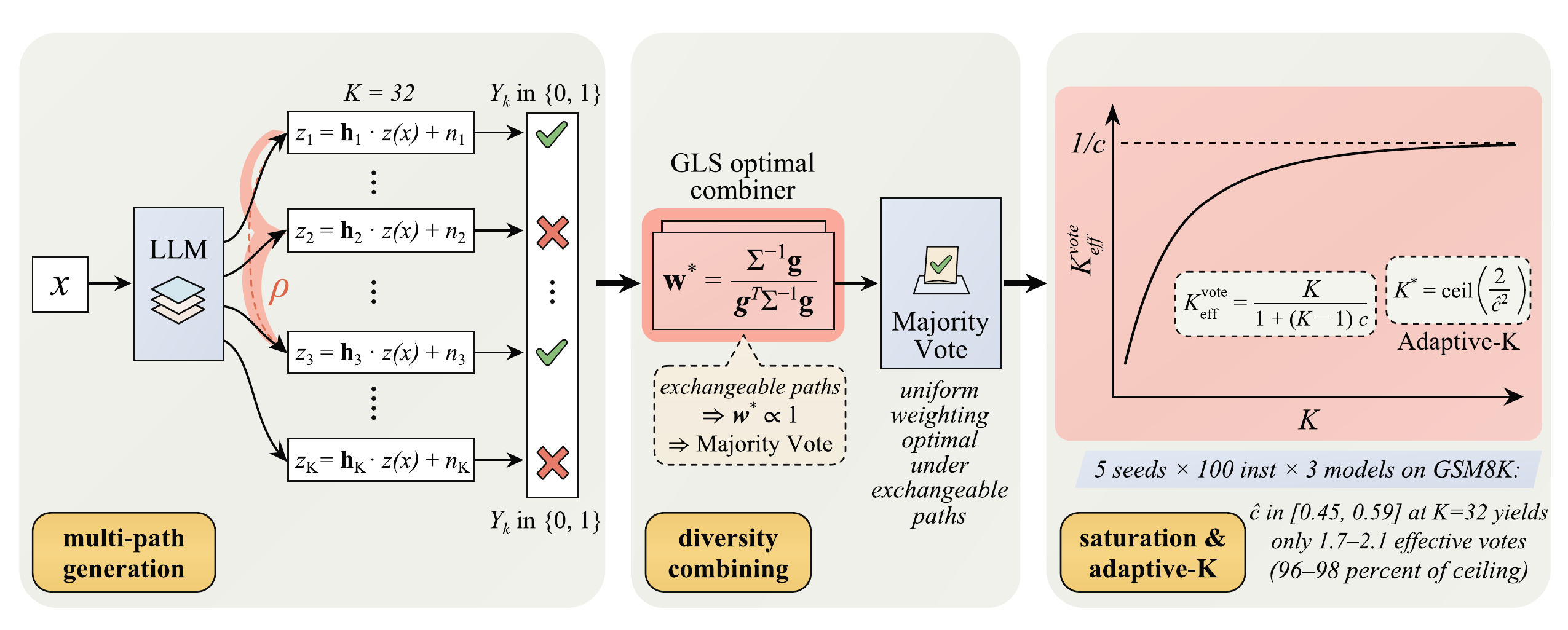}
\caption{Diversity combining for multi-path LLM reasoning. The LLM generates $K$ correlated paths $\mathbf{z}_k = h_k\mathbf{z}(x) + \mathbf{n}_k$ ($h_k$: reasoning quality; $\mathbf{n}_k$: errors; $\rho$: pairwise correlation), collapsed to correctness indicators $Y_k$. \textbf{Aggregation}: under exchangeability, the GLS-optimal linear combiner is uniform (Corollary~\ref{cor:special}), supporting uniform majority vote as the natural default. \textbf{Diagnostic}: $\Keff^{\text{vote}} = K/(1+(K{-}1)c)$ saturates at $1/c$; Adaptive-K estimates $K^*$ from a $K{=}4$ pilot.}
\label{fig:framework}
\end{figure*}

\section{Background and Related Work}
\label{sec:background}

\smallskip\noindent\textbf{Multi-path LLM reasoning.}
Self-consistency~\cite{wang2023selfconsistency} generates $K$ reasoning paths by sampling at temperature $\tau > 0$ and returns the plurality answer.
Extensions include tree-structured search~\cite{yao2023tot} and token-level embedding aggregation~\cite{softcot2025,softthinking2025}.
A parallel line on efficient reasoning~\cite{tokenbudget2024,budgetthinker2025,stopoverthinking2025,snell2024scaling} studies when and how to reduce multi-path compute, showing that adaptive scaling outperforms brute-force path multiplication.
Feng et al.~\cite{feng2025optimalsc} derive power-law scaling of SC from mode estimation and allocate samples across questions dynamically. Concurrently, Bay and Yearick~\cite{bay2026modalceiling} describe a modal ceiling and a correlation ceiling of test-time scaling and summarize the cutoff by an effective number of samples. We derive this effective number in closed form from the pairwise correctness correlation and set $K^*$ from a four-path pilot.
Training-based approaches change how paths are generated. Global forking tokens~\cite{jia2026forking} train models toward diverse yet correct reasoning modes, and Native Parallel Reasoner~\cite{wu2026nativeparallel} trains models to reason in parallel branches within a single response. Our diagnostic measures the path correlation that such train-time interventions act on (Appendix~\ref{app:comparison}).

\smallskip\noindent\textbf{Diversity combining in multipath channels.}
In wireless communications, a receiver observes $K$ correlated copies of a transmitted signal through distinct paths with gain and additive noise; classical combiners (selection, maximal-ratio, MMSE) trade off complexity against correlation awareness~\cite{tse2005fundamentals}.
For an equally-correlated model with coefficient $\rho$, the latent effective rank scales as $K/(1+(K{-}1)\rho^2)$~\cite{simon2005digital}; the corresponding vote-level effective sample size $K/(1+(K{-}1)c)$ divides $K$ by the Kish design effect $1+(K{-}1)c$~\cite{kish1965}.

\smallskip\noindent\textbf{Positioning.}
Existing efficient-SC work treats path-budget reduction as either online stopping by vote agreement or quality (Adaptive-Consistency~\cite{aggarwal2023adaptive}, ESC~\cite{li2024esc}, RASC~\cite{wan2025rasc}), confidence-weighted aggregation (CISC~\cite{cisc2025}, \cite{entropyvoting2025,optimalagg2025}), or compound-inference scaling-law fitting (\cite{chen2024morellmcalls}, large-scale repeated sampling~\cite{brown2024monkeys,snell2024scaling}). We organize these under a single \emph{upstream} diagnostic: pairwise correctness correlation $c$ limits effective diversity to $K/(1{+}(K{-}1)c)$, supplying a closed-form ceiling $1/c$ and a pilot-estimated $K^*$. Because $\hat{c}$ equals the corrected between-instance variance of per-instance accuracy divided by $\bar{p}(1-\bar{p})$, it summarizes in one vote-level statistic the difficulty heterogeneity that query-difficulty scaling models fit as a mixture. Table~\ref{tab:comparison} (Appendix~\ref{app:comparison}) summarizes the per-method differences.

\section{System Model}
\label{sec:model}

A summary of notation used throughout the paper is provided in Table~\ref{tab:notation} (Appendix~\ref{app:notation}).

\subsection{The Reasoning Channel}

We use a communications-inspired \emph{effective} model as an analytically tractable abstraction, not as a claim about the internal mechanism of language generation.
Reasoning paths generated by the same model under the same prompt share that model's competence on each problem, so their correctness is correlated across problems even when the paths for a given problem are sampled independently.
In this view, the ground-truth answer is the latent target, each reasoning path is a branch observation, and the final aggregator is a receiver-side combiner.
The value of the model lies in the predictions it enables (saturation law, majority-vote calibration) and the correlation diagnostic it motivates, which we validate empirically in \S\ref{sec:experiments}.
The saturation law, the beta-binomial prediction and Adaptive-K use observable correctness alone. Among the theoretical results only the latent-rank part of Theorem~\ref{thm:keff} uses \eqref{eq:channel}, and the GLS results use the separate linear model of \S\ref{sec:gls} (Appendix~\ref{app:assumption_map}).

Consider an LLM generating $K$ reasoning paths for a problem with ground-truth answer $x \in \mathcal{A}$, where $\mathcal{A}$ is a finite answer space.
Let $\mathbf{z}(x) \in \mathbb{R}^D$ denote the latent embedding of the correct answer.
Each reasoning path $k$ produces a latent representation:
\begin{equation}
\label{eq:channel}
\mathbf{z}_k = h_k \, \mathbf{z}(x) + \mathbf{n}_k, \quad k = 1, \ldots, K,
\end{equation}
where $h_k > 0$ is a real-valued channel gain capturing the reasoning quality of path $k$, and $\mathbf{n}_k \sim \mathcal{N}(\mathbf{0}, \sigma_n^2 \mathbf{I}_D)$ represents reasoning noise (hallucination, arithmetic errors).

The final answer is decoded from a combined representation $\hat{\mathbf{z}}$ as $\hat{x} = \argmin_{a \in \mathcal{A}} \| \hat{\mathbf{z}} - \mathbf{z}(a) \|^2$.
This decoding step is part of the analytical abstraction only, and no latent-embedding decoding is performed. The implemented pipeline extracts discrete answers from generated text. Deployment returns the plurality vote over these answers (Algorithm~\ref{alg:adaptive_k}), and the reported accuracies are the binary majority vote on their gold-scored correctness (\S\ref{sec:setup_mv}). The empirical GLS-R combiner (\S\ref{sec:gls}) is the only construct defined on embeddings, and it requires hidden-state access.

\begin{assumption}[Exchangeable paths]
\label{asm:exchangeable}
Channel gains $h_1, \ldots, h_K$ are identically distributed with $\mathbb{E}[h_k] = \mu_h > 0$ and $\Var(h_k) = \sigma_h^2$.
Noise vectors $\mathbf{n}_k$ are i.i.d.\ across paths and independent of $h_k$.
\end{assumption}

This assumption is more likely to hold when all $K$ paths are generated by the same model with the same prompt and sampling temperature.

\begin{remark}[Evaluation-induced binary collapse]
\label{rem:binary_collapse}
For open-form tasks (HotpotQA, TriviaQA, DROP), we operate on the binary-collapsed correctness indicator $Y_k = \mathbf{1}\{a_k = x\}$, reducing any task to $\mathcal{A} = \{0,1\}$ so that the majority-vote (Theorem~\ref{thm:mv}) and beta-binomial \eqref{eq:mv_corr} results apply regardless of the original output format. Scope, thresholding, and GLS exclusion are discussed in Appendix~\ref{app:binary_collapse}.
\end{remark}

\subsection{Path Correlation}
\label{sec:path_corr}

Since all paths originate from the same model and prompt, the channel gains $h_1, \ldots, h_K$ are correlated.
We model this correlation through an equally-correlated structure for the \emph{path correlation matrix} $\mathbf{R} \in \mathbb{R}^{K\times K}$:
\begin{equation}
\label{eq:corr_matrix}
R_{ij} = \begin{cases} \sigma_h^2 & i = j, \\ \sigma_h^2 \rho & i \neq j, \end{cases}
\end{equation}
where $\rho \in [0, 1]$ is the pairwise path correlation coefficient. 

\smallskip\noindent\textbf{Estimating path correlation from data.}
The latent correlation $\rho$ is not directly observable.
We estimate path dependence through the \emph{correctness correlation}
\begin{equation}
\label{eq:corr_c}
\hat{c}_{jk} = \frac{\frac{1}{n}\sum_{i=1}^{n} Y_i^{(j)} Y_i^{(k)} - \bar{p}^2}{\bar{p}(1 - \bar{p})},
\end{equation}
where $Y_i^{(k)} = \mathbf{1}\{a_i^{(k)} = x_i\}$ is the correctness indicator for path $k$ on instance $i$, and $\bar{p} = \frac{1}{nK}\sum_{i,k} Y_i^{(k)}$ is the empirical mean accuracy.
This is the natural estimator for the vote-level correlation that directly governs majority-vote behavior through $\Keff^{\text{vote}} = K / (1 + (K{-}1)c)$.
In the saturation and Adaptive-K experiments, we report $\hat{c} = \binom{K}{2}^{-1} \sum_{j<k} \hat{c}_{jk}$, averaged over all path pairs, and the cross-benchmark comparison reports the Pearson estimate $\hat{\rho}$ defined in \S\ref{sec:experiments}.
Pooled over instances, $\hat{c}$ equals the finite-$K$-corrected between-instance variance of per-instance accuracy divided by $\bar{p}(1-\bar{p})$.
In identically prompted SC the $K$ paths of an instance are sampled independently, so $\hat{c}$ is a marginal intraclass correlation, $\Keff^{\text{vote}}$ is the corresponding design-effect effective sample size, and $1/c$ is its asymptotic ceiling.
Under prompt templates the $K$ slots follow different distributions, and we use that setting to test whether $\hat{c}$ moves with the sampling configuration.
We use $\hat{c}$ rather than alternatives such as Cohen's kappa because $\hat{c}$ is the quantity that enters the effective sample size formula (see Appendix~\ref{app:kappa_comparison} for comparison).

\begin{assumption}[Gaussian decision surrogate]
\label{asm:gaussian_decision}
For analytical tractability, we model the binary correctness indicators $Y_1, \ldots, Y_K$ as arising from thresholding jointly Gaussian latent decision scores $M_1, \ldots, M_K$:
$Y_k = \mathbf{1}\{M_k > 0\}$, with $\Pr(M_k > 0) = p$ and $\Corr(M_j, M_k) = \rho_d$.
In the binary case $\mathcal{A} = \{x, a\}$, the nearest-neighbor decision margin $M_k = h_k \langle \mathbf{z}(x), \boldsymbol{\Delta} \rangle + \langle \mathbf{n}_k, \boldsymbol{\Delta} \rangle - b$ with $\boldsymbol{\Delta} = \mathbf{z}(x) - \mathbf{z}(a)$ and $b = (\|\mathbf{z}(x)\|^2 - \|\mathbf{z}(a)\|^2)/2$ is linear in the observation $\mathbf{z}_k$. When gains are jointly Gaussian, the margins are jointly Gaussian with $\rho_d = \kappa \rho$, where $\kappa = \sigma_h^2 \langle \mathbf{z}(x), \boldsymbol{\Delta} \rangle^2 / (\sigma_h^2 \langle \mathbf{z}(x), \boldsymbol{\Delta} \rangle^2 + \sigma_n^2 \|\boldsymbol{\Delta}\|^2) \in [0, 1]$ whenever the margin variance is positive, because the path noise is independent across paths and only inflates the variances. The noise therefore attenuates the gain correlation, and for $\rho > 0$ and $\sigma_h^2 \langle \mathbf{z}(x), \boldsymbol{\Delta} \rangle^2 > 0$ the equality $\rho_d = \rho$ holds only in the noiseless limit.
For $|\mathcal{A}| > 2$, the nearest-neighbor decision boundary is polyhedral and $\rho_d$ depends on both $\rho$ and the answer-space geometry; we treat $\rho_d$ as a surrogate parameter calibrated through the observable $\hat{c}$.
\end{assumption}

\subsection{Two Notions of Effective Diversity}

We distinguish two notions that must not be conflated.

\smallskip\noindent\textbf{Latent covariance effective rank.}
The participation ratio of the gain covariance matrix $\mathbf{R}$ measures the effective rank of the gain structure:
\begin{equation}
\label{eq:keff_rank}
\Keff^{\text{rank}}(\mathbf{R}) = \frac{(\tr \mathbf{R})^2}{\tr(\mathbf{R}^2)}.
\end{equation}
The participation ratio is scale-invariant, so $\Keff^{\text{rank}}$ depends only on $\rho$, not on $\sigma_h^2$.

\smallskip\noindent\textbf{Vote-level effective sample size.}
For binary correctness indicators $Y_k = \mathbf{1}\{a_k = x\}$ with pairwise correlation $c = \Corr(Y_i, Y_j)$, the standard design-effect formula gives:
\begin{equation}
\label{eq:keff_vote}
\Keff^{\text{vote}} = \frac{K}{1 + (K{-}1)c}.
\end{equation}

The latent effective rank characterizes covariance geometry; the vote-level effective sample size is the more direct object for predicting majority-vote accuracy.

\begin{proposition}[Bridge between latent and vote-level diversity]
\label{prop:bridge}
Both effective diversity measures share the form $K / (1 + (K{-}1)\gamma)$ with $\gamma = \rho^2$ (latent rank) or $\gamma = c$ (vote level), where $\gamma \mapsto K/(1+(K{-}1)\gamma)$ is monotone decreasing.
Under the Gaussian decision surrogate (Assumption~\ref{asm:gaussian_decision}) with $p \in (0,1)$, the binary correctness correlation $c = \Corr(Y_j, Y_k)$ is a strictly monotone increasing function of the decision-layer correlation $\rho_d$ (proof in Appendix~\ref{app:bridge_proof}).
The relationship between $c$ and $\rho^2$ depends on both $p$ and $\rho_d$; the two quantities are not directly comparable in general.
\emph{Assumption used:} Assumption~\ref{asm:gaussian_decision} only.
\end{proposition}

\smallskip\noindent\textbf{Practical consequence.}
Since the latent $\rho$ is unobservable, the saturation and Adaptive-K experiments use $\Keff^{\text{vote}} = K/(1+(K{-}1)\hat{c})$ with the directly measured $\hat{c}$, and the cross-benchmark comparison substitutes $\hat{\rho}$ for $\hat{c}$. Proposition~\ref{prop:bridge} ensures that higher latent dependence produces higher vote-level dependence, so the two effective-diversity measures move in the same direction (design-effect derivation~\cite{kish1965} in Appendix~\ref{app:design_effect}).

\section{Theoretical Analysis}
\label{sec:theory}

\subsection{Effective Rank of Correlated Paths}

\begin{theorem}[Effective rank under equicorrelation]
\label{thm:keff}
For $K$ equally-correlated paths with correlation coefficient $\rho$ as in \eqref{eq:corr_matrix},
\begin{equation}
\label{eq:keff_formula}
\Keff^{\text{rank}} = \frac{K}{1 + (K{-}1)\rho^2}, \qquad \Keff^{\text{rank}} \to \frac{1}{\rho^2} \text{ as } K \to \infty.
\end{equation}
The vote-level analog $\Keff^{\text{vote}} = K/(1+(K{-}1)c)$ saturates at $1/c$. Proof via eigenvalue decomposition of $\mathbf{R}$ is in Appendix~\ref{app:keff_proof}.
\emph{Assumptions used:} for $\Keff^{\text{rank}}$, \eqref{eq:channel} with Assumption~\ref{asm:exchangeable} and the equicorrelated matrix \eqref{eq:corr_matrix}. For $\Keff^{\text{vote}}$, equicorrelated correctness indicators.
\end{theorem}

\subsection{Majority Vote Accuracy}

We analyze majority vote through a binary-collapsed correctness process.
Let $Y_k = \mathbf{1}\{a_k = x\}$ where $a_k$ is the extracted answer from path $k$.

\begin{theorem}[Majority vote under independence]
\label{thm:mv}
If $Y_1, \ldots, Y_K$ are i.i.d.\ $\text{Bernoulli}(p)$, for odd $K$:
\begin{equation}
P_{\text{MV}}(K, p) = \sum_{j=(K+1)/2}^{K} \binom{K}{j} p^j (1{-}p)^{K-j}.
\end{equation}
For even $K$ with random tie-breaking, the tie term $\binom{K}{K/2} p^{K/2}(1{-}p)^{K/2}$ contributes a factor of $1/2$.
\emph{Assumption used:} i.i.d.\ correctness indicators (the independence reference).
\end{theorem}

This formula is exact for the binary-collapsed model.
In the original multiclass answer space with $|\mathcal{A}| > 2$, plurality self-consistency depends on the full wrong-answer distribution and the binomial formula is a surrogate.

\begin{remark}[Multiclass plurality heuristic]
\label{rem:multiclass}
Let $|\mathcal{A}| = M > 2$ and let $p$ denote the probability that a single path produces the correct answer.
Under uniform wrong-answer fragmentation among $M{-}1$ alternatives, the effective pairwise accuracy for correct-vs-most-popular-wrong is $p' = p(M{-}1) / (pM - 2p + 1)$.
Since $p' > p$ for $M > 2$, plurality voting can gain substantially over the binary-collapsed prediction. For $M{=}4$ and $p{=}0.5$, $p'{=}0.75$ (see Appendix~\ref{app:multiclass_proof} for the derivation).
The binary-collapsed accuracy is in any case a lower bound, $P_{\text{PV}}^{(M)} \ge P_{\text{MV}}^{(2)}$ on the same draws, for any wrong-answer distribution and any dependence among paths. A correct answer holding more than $K/2$ votes leaves fewer than $K/2$ for every other answer and therefore wins the plurality, and at an even-$K$ tie of $K/2$ votes it wins with probability at least $1/2$ under random tie-breaking.
\end{remark}

Non-uniform wrong-answer fragmentation (where some wrong answers are more popular) reduces the plurality advantage.

\smallskip\noindent\textbf{Correlated majority vote.}
When paths are correlated, we model the joint correctness distribution using the classical beta-binomial (BB) model~\cite{skellam1948}.
Assume $\Theta \sim \text{Beta}(\alpha, \beta)$ and $Y_k \mid \Theta \overset{\text{i.i.d.}}{\sim} \text{Bernoulli}(\Theta)$.
Then $Y_k$ are exchangeable with:
\begin{equation}
p = \frac{\alpha}{\alpha + \beta}, \quad c = \frac{1}{\alpha + \beta + 1}.
\end{equation}
The count $S_K = \sum_k Y_k$ follows a beta-binomial distribution, and the correlated majority-vote accuracy is:
\begin{equation}
\label{eq:mv_corr}
P_{\text{MV}}^{\text{corr}} = \sum_{j > K/2} \binom{K}{j} \frac{B(j{+}\alpha, K{-}j{+}\beta)}{B(\alpha, \beta)},
\end{equation}
where $B(\cdot, \cdot)$ is the beta function.
For even $K$ with random tie-breaking, the $j{=}K/2$ term contributes half its probability mass: add $\frac{1}{2}\binom{K}{K/2} B(K/2{+}\alpha,\, K/2{+}\beta) / B(\alpha, \beta)$ to \eqref{eq:mv_corr}.

\subsection{GLS-Optimal Combining}
\label{sec:gls}

Equal-weight combining ignores cross-path redundancy.
We now analyze a generalized linear aggregation model that shares the branch-combining structure of \S\ref{sec:model} but is not a direct corollary of Eq.~\eqref{eq:channel}.
The GLS results (Theorem~\ref{thm:gls}, Corollaries~\ref{cor:special}--\ref{cor:gain}) are the classical generalized least squares solution~\cite{aitken1935} and hold for any linear estimation model satisfying $\mathbf{e}_k = g_k \mathbf{s} + \boldsymbol{\varepsilon}_k$, independent of the channel model. Assume each path embedding follows this model for $k=1,\ldots,K$, where $\mathbf{s}\in\mathbb{R}^D$ is the latent answer embedding, $g_k>0$ is a path-quality gain, and for each feature dimension $d$, the branch error vector $\boldsymbol{\varepsilon}^{(d)}=(\varepsilon_{1d},\ldots,\varepsilon_{Kd})^\top$ has mean zero and covariance $\boldsymbol{\Sigma}\in\mathbb{R}^{K\times K}$.

For a linear estimator $\hat{\mathbf{s}}(\mathbf{w}) = \sum_k w_k \mathbf{e}_k$, unbiasedness requires $\mathbf{g}^\top \mathbf{w} = 1$.
The risk is $\mathbb{E}\|\hat{\mathbf{s}} - \mathbf{s}\|^2 = D \cdot \mathbf{w}^\top \boldsymbol{\Sigma} \mathbf{w}$.

\begin{theorem}[GLS-optimal combiner]
\label{thm:gls}
Let $\boldsymbol{\Sigma}$ be positive definite.
Among all linear unbiased estimators satisfying $\mathbf{g}^\top \mathbf{w} = 1$, the minimum-MSE weights are:
\begin{multline}
\label{eq:gls_weights}
\mathbf{w}_* = \frac{\boldsymbol{\Sigma}^{-1} \mathbf{g}}{\mathbf{g}^\top \boldsymbol{\Sigma}^{-1} \mathbf{g}}, \\ \text{with optimal risk}  \MSE(\mathbf{w}_*) = D / (\mathbf{g}^\top \boldsymbol{\Sigma}^{-1} \mathbf{g}).
\end{multline}
\emph{Assumption used:} the linear embedding model of \S\ref{sec:gls}, independent of \eqref{eq:channel}.
\end{theorem}

The Lagrangian derivation and strict-improvement conditions are given in Appendix~\ref{app:gls_proof}.

\begin{corollary}[When uniform weighting is optimal]
\label{cor:special}
With equal gains $\mathbf{g} = \mathbf{1}$, $\mathbf{w}_* = (1/K)\mathbf{1}$ if and only if $\boldsymbol{\Sigma}\mathbf{1} \propto \mathbf{1}$, \ie every row of $\boldsymbol{\Sigma}$ has the same sum, because $\mathbf{w}_* \propto \boldsymbol{\Sigma}^{-1}\mathbf{1}$ and $\boldsymbol{\Sigma}^{-1}\mathbf{1} \propto \mathbf{1} \Leftrightarrow \boldsymbol{\Sigma}\mathbf{1} \propto \mathbf{1}$. The exchangeable model $\boldsymbol{\Sigma} = \sigma^2 [(1{-}\rho)\mathbf{I} + \rho \mathbf{1}\mathbf{1}^\top]$ satisfies this condition, and so does any covariance whose unequal variances and correlations balance within each row. Other special cases (white noise $\Rightarrow$ Maximum Ratio Combining (MRC), and equal gains $\Rightarrow$ covariance-aware averaging) are given in Appendix~\ref{app:gls_proof}.
\emph{Assumption used:} as Theorem~\ref{thm:gls}, with equal gains.
\end{corollary}

\begin{remark}[GLS$\rightarrow$majority vote]
\label{rem:gls_to_mv}
Corollary~\ref{cor:special} establishes that, with equal path gains, uniform weighting is optimal among linear combiners operating on the latent embeddings exactly when the path covariance has equal row sums.
Majority vote operates on the decoded discrete answers, and the two can disagree even under the exchangeable model, because decoding the uniform average of the embeddings need not return the most frequent decoded answer.
We interpret the GLS result as theoretical grounding for the empirical effectiveness of majority vote, not as a claim that MV is Bayes-optimal among all aggregation rules.
More precisely, GLS minimizes mean-squared error over continuous latent estimates, whereas majority vote selects a discrete answer evaluated by classification accuracy.
These are different optimization problems, so the GLS optimality does not transfer to majority vote.
\end{remark}

\begin{corollary}[Gain over uniform averaging]
\label{cor:gain}
$\MSE(\mathbf{w}_*) \le \MSE(\mathbf{w}_{\text{unif}})$, with equality iff $\boldsymbol{\Sigma} \mathbf{1} \propto \mathbf{1}$. Heterogeneous path variances or correlations therefore improve on uniform averaging only when they make the row sums of $\boldsymbol{\Sigma}$ unequal (proof in Appendix~\ref{app:gain_proof}).
\emph{Assumption used:} as Theorem~\ref{thm:gls}, with equal gains $\mathbf{g} = \mathbf{1}$ so that uniform weights are feasible.
\end{corollary}

\smallskip\noindent\textbf{Empirical GLS-R.}
In practice, the branch covariance $\boldsymbol{\Sigma}$ is estimated from a calibration set.
Given $T$ problems with path embeddings, we form centered residuals $\tilde{\mathbf{E}}_t \in \mathbb{R}^{K \times D}$ and estimate $\hat{\boldsymbol{\Sigma}} = (TD)^{-1} \sum_t \tilde{\mathbf{E}}_t \tilde{\mathbf{E}}_t^\top$.
The regularized empirical GLS weights are:
\begin{equation}
\label{eq:gls_empirical}
\hat{\mathbf{w}}_\lambda = \frac{(\hat{\boldsymbol{\Sigma}} + \lambda \mathbf{I})^{-1} \hat{\mathbf{g}}}{\hat{\mathbf{g}}^\top (\hat{\boldsymbol{\Sigma}} + \lambda \mathbf{I})^{-1} \hat{\mathbf{g}}}.
\end{equation}

\section{Experiments}
\label{sec:experiments}

The experiments \emph{validate} the framework's predictions (saturation ceiling, BB calibration, uniform-weighting optimality) and \emph{characterize} how prompt-template perturbations probe the correlation structure of multi-path reasoning across 12 benchmarks spanning six domains.

\subsection{Experimental Settings}

\smallskip\noindent\textbf{Models.}
We evaluate five instruction-tuned models across three architecture families: Qwen2.5-0.5B/7B/32B-Instruct, Llama-3.1-8B-Instruct, and Mistral-7B-Instruct-v0.3. Saturation analysis (\S\ref{sec:saturation}) uses Qwen-7B, Llama-8B, Mistral-7B at $K \in \{4, 8, 16, 32\}$; cross-benchmark analysis (\S\ref{sec:cross_bench}) uses all five at $K{=}8$. A reasoning model, Qwen3.5-9B in thinking mode, is evaluated separately (Appendix~\ref{app:reasoning}).

\smallskip\noindent\textbf{Benchmarks.}
We evaluate on 12 benchmarks spanning six task domains (Math, QA, Sci/MC, Code, Commonsense, NLU); the full list with citations and the per-benchmark \emph{effective evaluated output space} (numeric, open text, bounded choice, binary, code pass/fail) is in Appendix~\ref{app:benchmarks}, with the answer-space label appearing as a column of Table~\ref{tab:full_matrix}.
Saturation experiments (\S\ref{sec:saturation}--\S\ref{sec:cross_arch}) use $n{=}100$ instances per seed on GSM8K, pooled across $5$ seeds ($500$ per cell). Cross-benchmark experiments (\S\ref{sec:cross_bench}) use the same $n{=}100$ instances per seed, pooled across $5$ seeds for GSM8K, MATH, TriviaQA, DROP, CruxEval and MBPP and across $3$ seeds for the other six benchmarks ($500$ or $300$ instances per cell).

\smallskip\noindent\textbf{Methods and metrics.}
Saturation analysis (\S\ref{sec:saturation}--\S\ref{sec:cross_arch}) runs standard self-consistency (SC, temperature $\tau{=}0.7$) at $K \in \{4, 8, 16, 32\}$ and reports the correctness correlation $\hat{c}$ from \eqref{eq:corr_c}, the \emph{modeled} quantity entering the design-effect formula.
Cross-benchmark comparisons (\S\ref{sec:cross_bench}) contrast SC ($K{=}8$, the first eight of $32$ sampled paths) against prompt-template SC (PT, $K{=}8$ structurally distinct templates per task, such as algebraic vs.\ estimation for math and chain-of-thought vs.\ extract-then-answer for QA, with the full set in Appendix~\ref{app:templates}), reporting $\hat{\rho}$, the \emph{measured} mean pairwise Pearson correlation of the $K$-path binary correctness vectors (the two estimators agree to within $0.004$ under SC and to within $0.06$ under PT, where per-template accuracies differ, and with $\hat{c}$ the same 55 cells decorrelate), with $\Delta\rho = (\hat{\rho}_{\text{PT}} - \hat{\rho}_{\text{SC}}) / |\hat{\rho}_{\text{SC}}|$.
Binary correctness comes from task-specific scorers on the extracted final answer, identical in both arms: numeric equivalence for GSM8K, canonical exact match for MATH, alias exact match for TriviaQA and DROP, token-level F1 $\ge 0.5$ for HotpotQA, the option letter for the multiple-choice tasks, yes/no for BoolQ, literal equality for CruxEval, and passing every reference assertion in a sandbox for MBPP.
Throughout, MV@$K$ is the binary majority vote on these indicators\label{sec:setup_mv}: an instance scores $1$ when more than $K/2$ of its paths are correct and $1/2$ at an exact tie, matching the decision rule of \eqref{eq:mv_corr}.

\smallskip\noindent\textbf{Implementation details.}
All experiments use Hugging Face Transformers with bfloat16 precision on NVIDIA H100 GPUs.
Maximum generation length is set per benchmark and shared by both arms: 256 tokens for BoolQ, 512 for HotpotQA, TriviaQA, DROP, HellaSwag and WinoGrande, 1024 for GSM8K, ARC-Challenge, CruxEval and MBPP, 2048 for MATH and 4096 for MMLU. The saturation and Adaptive-K experiments with non-reasoning models use the same budgets for GSM8K, HotpotQA and BoolQ, and the reasoning model of Appendix~\ref{app:reasoning} uses 16{,}384 tokens.

\subsection{Experimental Results}

\subsubsection{Diversity Saturation in Self-Consistency}
\label{sec:saturation}

We use GSM8K as the primary saturation benchmark because it is the principled stress test for the diagnostic: closed-form numeric answers admit a clean binary correctness collapse and isolate the correlation parameter $\hat{c}$ from partial-credit or open-form scoring confounds; the cross-task results in Table~\ref{tab:adaptive_k} (HotpotQA $\hat{c}{=}0.61$, BoolQ $\hat{c}{=}0.79$) further show that the saturation rule remains predictive in higher-correlation regimes spanning open-form QA and binary NLU at $K{=}\{4, 8, 16, 32\}$ for those rows; a per-$K$ saturation sweep across additional open-form benchmarks is left for future work.

Table~\ref{tab:diversity} reports the effective diversity analysis for Qwen2.5-7B-Instruct on GSM8K.
As $K$ increases from 4 to 32, the correctness correlation $\hat{c}$ remains stable near 0.59, indicating persistent inter-path correlation.
The vote-level effective diversity $\Keff^{\text{vote}}$ saturates at 1.67 at $K{=}32$, reaching 98\% of its theoretical ceiling $1/\hat{c} = 1.71$.

\begin{table}[!htbp]
\caption{Effective diversity analysis for Qwen2.5-7B on GSM8K (100 instances, 5 seeds pooled). Agree is the mean pairwise agreement $\binom{K}{2}^{-1}\sum_{j<k}\Pr(Y^{(j)}{=}Y^{(k)})$. Correctness correlation $\hat{c}$ is computed via \eqref{eq:corr_c}. $\Keff^{\text{vote}}$ is the design-effect \eqref{eq:keff_vote}, and Ceiling $= 1/\hat{c}$. $\bar{p}$ is the mean per-path accuracy, reported as per-seed mean $\pm$ std over 5 seeds, and is not the majority-vote accuracy MV@$K$ of Table~\ref{tab:cross_arch}. The $K{=}1$ row is a single sampled path.}
\label{tab:diversity}
\centering
\vspace{5pt}
\footnotesize
\setlength{\tabcolsep}{2.5pt}
\begin{tabular}{cccccccc}
\toprule
$K$ & Agree & $\hat{c}$ & $\Keff^{\text{vote}}$ & Ceiling & \% Ceil. & $\bar{p}$ (\%) $\uparrow$ & Tokens $\downarrow$ \\
\midrule
1 & -- & -- & 1.0 & -- & -- & 78.0 $\pm$ 2.5 & 323 \\
4 & 0.868 & 0.604 & 1.42 & 1.66 & 86\% & 79.0 $\pm$ 1.3 & 1293 \\
8 & 0.866 & 0.593 & 1.55 & 1.69 & 92\% & 79.2 $\pm$ 0.8 & 2592 \\
16 & 0.862 & 0.579 & 1.65 & 1.73 & 96\% & 79.4 $\pm$ 0.8 & 5179 \\
32 & 0.864 & 0.586 & 1.67 & 1.71 & 98\% & 79.2 $\pm$ 0.4 & 10327 \\
\bottomrule
\end{tabular}%
\end{table}

The rapid ceiling approach confirms the design-effect prediction. $\Keff^{\text{vote}} \ll K$ and grows only from $1.55$ to $1.67$ between $K{=}8$ and $K{=}32$, so paths beyond $K{=}8$ add little independent evidence to the vote.

\subsubsection{Cross-Architecture Validation}
\label{sec:cross_arch}

\begin{table}[!htbp]
\caption{Saturation and beta-binomial calibration on GSM8K (Qwen2.5-7B, Llama-3.1-8B, Mistral-7B, 100 instances, 5 seeds pooled per cell). $\Keff^{\text{vote}} = K/(1+(K{-}1)\hat{c})$. All models reach $96$--$98\%$ of the ceiling $1/\hat{c}$ at $K{=}32$. MV@$K$ denotes the binary majority-vote accuracy at $K$ paths (\S\ref{sec:setup_mv}), which is distinct from the mean per-path accuracy $\bar{p}$ of Table~\ref{tab:diversity}. BB (beta-binomial) and Binom (independence-binomial) columns are in-sample MV@32 predictions, and both estimators are formally defined in the next subsection (Eq.~\eqref{eq:mv_corr}).}
\label{tab:cross_arch}
\centering
\vspace{5pt}
\footnotesize
\setlength{\tabcolsep}{1.9pt}
\begin{tabular}{lc cccc c cc}
\toprule
 & & \multicolumn{4}{c}{$\Keff^{\text{vote}}$ at $K{=}$} & & \multicolumn{2}{c}{Predicted MV@32} \\
\cmidrule(lr){3-6} \cmidrule(lr){8-9}
Model & $\hat{c}_{K{=}4}$ & 4 & 8 & 16 & 32 & MV@32 & BB & Binom \\
\midrule
Qwen2.5-7B   & 0.60 & 1.42 & 1.55 & 1.65 & 1.67 & 81.9\% & 81.1\% & 100.0\% \\
Llama-3.1-8B & 0.53 & 1.55 & 1.72 & 1.86 & 1.94 & 79.3\% & 77.8\% &  99.9\% \\
Mistral-7B   & 0.45 & 1.71 & 1.89 & 2.00 & 2.13 & 42.4\% & 41.3\% &  20.6\% \\
\bottomrule
\end{tabular}%
\end{table}

Table~\ref{tab:cross_arch} validates the vote-level saturation law across three model families spanning base accuracies of $43$--$79\%$.
Mistral has the lowest correctness correlation ($\hat{c}{=}0.45$), yielding the highest ceiling ($1/\hat{c}{=}2.21$) and the highest $\Keff^{\text{vote}}$ at $K{=}32$ (2.13).
The saturation law is model-agnostic: it depends only on $c$, not on the absolute accuracy level.

\subsubsection{Beta-Binomial Calibration}
\label{sec:bb_calibration}

We validate the correlated majority-vote model from \eqref{eq:mv_corr} (derivation in Appendix~\ref{app:betabin}) by fitting beta-binomial parameters $(\alpha, \beta)$ via method of moments (MoM) from the observed vote-count distribution, then comparing the predicted MV accuracy against empirical values.
We use MoM rather than maximum-likelihood or GLMM estimation because it yields closed-form $(\alpha, \beta)$ from $(p, c)$ alone, matching the pilot-based diagnostic use case; at $n{=}100$, MoM and MLE estimates are comparable for the overdispersion levels observed here.

Table~\ref{tab:cross_arch} shows that the independence-assuming binomial diverges catastrophically at $K{=}32$ (e.g., $100.0\%$ predicted vs.\ $81.9\%$ actual for Qwen; $99.9\%$ vs.\ $79.3\%$ for Llama; $20.6\%$ vs.\ $42.4\%$ for Mistral), while the beta-binomial predicts within $0.8$--$1.5$ pp of the observed MV accuracy.

\smallskip\noindent\textbf{Held-out prediction test.}
When $(\alpha, \beta)$ are fitted from a $K{=}4$ pilot on one half of the GSM8K instances and used to predict $K \in \{8, 16, 32\}$ on the other half, BB absolute error at $K{=}32$ is $3.3$--$4.8$ pp across all three models (averaged over both halves), whereas the binomial diverges by $18$--$24$ pp (Fig.~\ref{fig:bb_calibration}).

\subsubsection{Prompt-Template Diversity Across 12 Benchmarks}
\label{sec:cross_bench}

The value of prompt-template diversity lies in the structure it reveals: its success criterion is whether the diagnostic variable moves under a controlled prompt change on a fixed question set. This subsection measures that \emph{decorrelation}, which raises the saturation ceiling $1/c$. Decorrelation by itself does not raise base path accuracy $\bar{p}$, and observed MV accuracy changes are correspondingly small and mixed in sign (Appendix~\ref{app:pt_acc_delta}).
Temperature-only sampling yields exchangeable correctness patterns with limited decorrelation (Corollary~\ref{cor:special}); to break exchangeability, we assign structurally distinct prompt templates to each of the $K{=}8$ slots and measure $\Delta\rho$ against SC on the \emph{same model} across all 12 benchmarks.
No cell falls below the $2\%$ base-accuracy cut, and PT reduces $\hat{\rho}$ in 55 of the 60 model-benchmark cells (Fig.~\ref{fig:drho_cells}, with per-benchmark means in Table~\ref{tab:full_matrix} and all cells in Table~\ref{tab:cell_matrix}, Appendix~\ref{app:cell_matrix}), with mean $\Delta\rho{=}{-}16\%$ (median ${-}15\%$) and mean $\Delta\Keff^{\text{vote}}{=}{+}0.29$.
Four of the five exceptions are small, with $\Delta\rho$ between ${+}1\%$ and ${+}13\%$ (Llama-8B on ARC-C, Mistral-7B on HellaSwag, Qwen-0.5B on GSM8K and MATH).
The fifth, Qwen-0.5B on WinoGrande ($\Delta\rho{=}{+}90\%$), sits at chance accuracy ($\bar{p}{=}49\%$ on two options). There the templates push the slots toward the same option letter (B in $77\%$ of PT answers against a balanced gold), and the shared position bias correlates the slots.

\begin{figure*}[!t]
  \centering
  \begin{minipage}[t]{0.34\textwidth}
    \centering
    \includegraphics[width=0.88\linewidth]{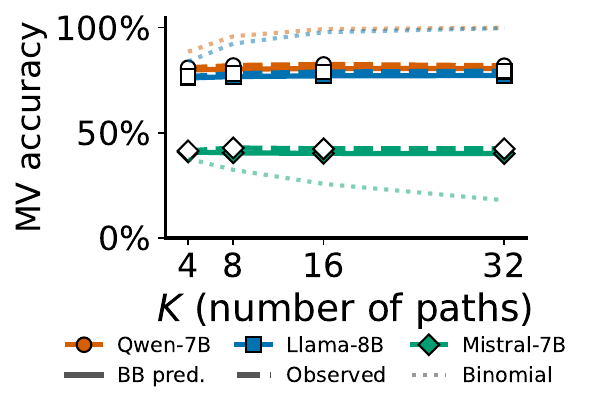}
    \caption{Held-out BB prediction on GSM8K: $(\alpha, \beta)$ fitted from $K{=}4$. BB tracks within $3.3$--$4.8$ pp at $K{=}32$ and the binomial diverges by $18$--$24$ pp.}
    \label{fig:bb_calibration}
  \end{minipage}\hfill
  \begin{minipage}[t]{0.62\textwidth}
    \centering
    \includegraphics[width=0.92\linewidth]{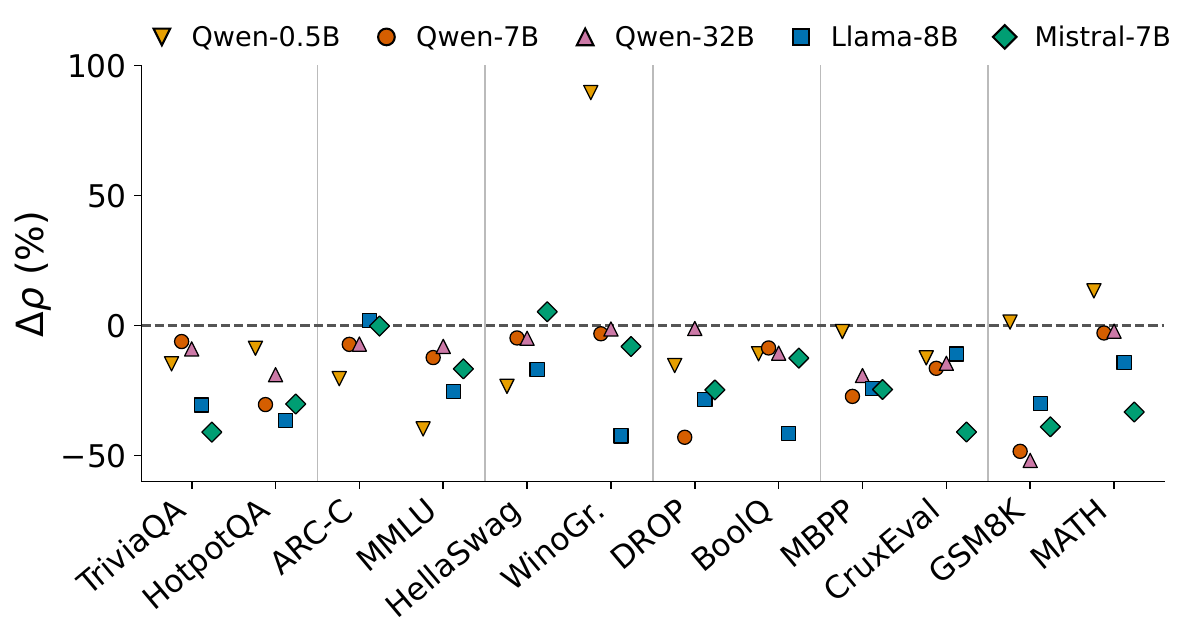}
    \caption{Per-cell PT decorrelation $\Delta\rho$ for all 60 model-benchmark cells, grouped by domain. The dashed line marks no change.}
    \label{fig:drho_cells}
  \end{minipage}
\end{figure*}

\smallskip\noindent\textbf{Decorrelation is broad, not answer-space-gated.}
Domain medians of $\Delta\rho$ range from ${-}5\%$ (Commonsense) to ${-}25\%$ (QA), and Math (median ${-}22\%$) decorrelates about as much as QA. A Mann-Whitney test of the Math cells against the QA cells ($n{=}10$ each, the unit being the model-benchmark cell) finds no difference ($p{=}0.73$), and answer diversity under SC does not predict $\Delta\rho$ across benchmarks ($r{=}{-}0.34$, $p{=}0.28$, $n{=}12$).
Path correlation therefore moves under a controlled prompt change on every task type, and the size of the move is a property of the model-task cell. Per-cell bootstrap intervals are reported in Appendix~\ref{app:statistical_precision}.
This decorrelation is not visible from accuracy curves or aggregation weights alone, which makes path correlation the diagnostic, and we use prompt-template perturbation as a structural probe of $c$ rather than as a replacement for SC.

\smallskip\noindent\textbf{Heterogeneous slot regime.}
When prompt templates create heterogeneous per-slot accuracy, the exchangeability of Corollary~\ref{cor:special} breaks and the GLS analysis leaves room for non-uniform weighting.
Consistent with Corollary~\ref{cor:gain}, the gain of in-sample accuracy-weighted MV over uniform MV on PT cells tracks the slot-accuracy coefficient of variation (CV, $r{=}0.85$ over 60 cells), averaging ${+}4.4$ pp in the seven cells with CV $> 0.225$ and ${+}1.1$ pp in the rest. Full results and the corresponding SC comparison are in Appendix~\ref{app:hetero_slot}.

\subsubsection{Adaptive-K}
\label{sec:adaptive_k}

The saturation law motivates a compute-allocation rule.
The marginal diversity gain from adding one path is $\partial \Keff^{\text{vote}} / \partial K = (1{-}c) / (1 + (K{-}1)c)^2$, which decreases in $K$.
Setting this marginal gain to a threshold $\varepsilon$ and solving yields:
\begin{equation}
\label{eq:adaptive_k}
K^* = \left\lceil \frac{\sqrt{(1{-}c)/\varepsilon} - 1}{c} + 1 \right\rceil.
\end{equation}
For the default threshold $\varepsilon{=}0.025$ (2.5\% marginal gain) and a $K{=}4$ pilot estimate $\hat{c}$, the shorthand $\lceil 2/\hat{c}^2 \rceil$ lies within one path of \eqref{eq:adaptive_k}, and never above it, for typical $\hat{c} \in [0.5, 0.7]$.
Algorithm~\ref{alg:adaptive_k} (Appendix~\ref{app:algorithm}) lists the full procedure.

\begin{table*}[!t]
\caption{Adaptive-K compute savings across three domains. $\hat{c}$ is estimated from $K{=}4$ SC ($n{=}100$ instances, 5 seeds pooled for GSM8K and 3 seeds pooled for HotpotQA and BoolQ). Retained $=$ MV accuracy at $K^*$ as a percentage of MV at $K{=}32$. Net cost $=(K^*+4)/32$ accounts for the $K{=}4$ pilot in addition to the operating-point paths.}
\label{tab:adaptive_k}
\centering
\vspace{5pt}
\footnotesize
\begin{tabular}{llcccccc}
\toprule
Task & Model & $\hat{c}_{K{=}4}$ & $K^*$ & MV@$K^*$ & MV@32 & Retained & Net cost \\
\midrule
GSM8K (Math)     & Qwen-7B     & 0.60 & 6 & 81.6\% & 81.9\% & 100\% & 31\% \\
GSM8K (Math)     & Llama-8B    & 0.53 & 8 & 78.1\% & 79.3\% &  98\% & 38\% \\
GSM8K (Math)     & Mistral-7B  & 0.45 & 10 & 43.6\% & 42.4\% & 103\%$^\dagger$ & 44\% \\
\midrule
HotpotQA (QA)    & Llama-8B    & 0.61 & 6 & 50.2\% & 52.2\% &  96\% & 31\% \\
BoolQ (NLU)      & Llama-8B    & 0.79 & 4 & 80.2\% & 80.3\% & 100\% & 25\% \\
\bottomrule
\multicolumn{8}{l}{\scriptsize $^\dagger$Retention above $100\%$ is consistent with finite-sample variance: the paired bootstrap}\\
\multicolumn{8}{l}{\scriptsize \phantom{$^\dagger$}95\% CI of MV@$K^*{-}$MV@32 contains zero in every cell.}
\end{tabular}%
\end{table*}

Table~\ref{tab:adaptive_k} shows that a $K{=}4$ pilot reliably predicts saturation across Math, QA, and NLU: $K^*$ retains $96$--$103\%$ of MV@$K{=}32$ at net inference cost $25$--$44\%$ of full $K{=}32$ SC, with higher-correlation tasks saturating earlier (the single $>$100\% cell is annotated by the dagger). The rule is also stable to $\varepsilon$: varying it from $0.01$ to $0.1$ changes $K^*$ but preserves $97$--$101\%$ accuracy on Qwen-7B GSM8K (Appendix~\ref{app:eps_sensitivity}).

\section{Conclusion}
\label{sec:conclusion}

This paper formalized multi-path LLM reasoning as a diversity combining problem.
The design-effect formula $\Keff^{\text{vote}} = K/(1{+}(K{-}1)c)$ predicts that 32 paths yield a design-effect effective sample size of only 1.7--2.1, a ceiling confirmed across three model families.
GLS analysis shows that, with equal path gains, the optimal linear combiner of latent embeddings is uniform exactly when the path covariance has equal row sums, as under exchangeability, supporting uniform majority vote as the natural default in standard SC and motivating slot-pruning and weighted-MV variants in the heterogeneous-template regime.
Prompt-template diversity reduces path correlation in 55 of 60 model-benchmark cells across 12 benchmarks and all six task domains (mean ${-}16\%$).
An Adaptive-K rule derived from the saturation law retains $96$--$103\%$ of accuracy across Math, QA, and NLU domains from a four-path pilot.

\begin{ack}
The authors received no specific funding for this work and declare no competing interests.
\end{ack}

\bibliography{refs}
\bibliographystyle{IEEEtran}

\appendix

\section{Notation}
\label{app:notation}

\begin{table}[h]
\caption{Notation used throughout the paper.}
\label{tab:notation}
\centering
\vspace{5pt}
\begin{threeparttable}
\begin{tabular}{cl}
\toprule
\textbf{Symbol} & \multicolumn{1}{c}{\textbf{Meaning}} \\
\midrule
\cellcolor{pink!0} $x \in \mathcal{A}$ & Ground-truth answer in discrete answer space. \\
\cellcolor{pink!0} $\mathbf{z}(x) \in \mathbb{R}^D$ & Latent embedding of answer $x$. \\
\cellcolor{pink!0} $K$ & Number of reasoning paths. \\
\cellcolor{pink!0} $h_k$ & Channel gain of path $k$ ($h_k > 0$). \\
\cellcolor{pink!0} $\mathbf{n}_k$ & Additive noise vector of path $k$. \\
\cellcolor{pink!0} $\rho \in [0,1]$ & Latent channel gain correlation (in $\mathbf{R}$). \\
\cellcolor{pink!0} $p$ & Single-path accuracy $\Pr(Y_k = 1)$. \\
\midrule
\cellcolor{yellow!0} $\Keff^{\text{rank}}$ & Latent effective rank (participation ratio). \\
\cellcolor{yellow!0} $\Keff^{\text{vote}}$ & Vote-level effective sample size (design effect). \\
\cellcolor{yellow!0} $c$ & Pairwise correctness correlation $\Corr(Y_i, Y_j)$. \\
\cellcolor{yellow!0} $\mathbf{R}$ & Gain covariance matrix ($K \times K$). \\
\cellcolor{yellow!0} $\boldsymbol{\Sigma}$ & Branch error covariance ($K \times K$). \\
\cellcolor{yellow!0} $\mathbf{g}$ & Path quality gain vector. \\
\bottomrule
\end{tabular}%
\begin{tablenotes}
\scriptsize
\item Upper: problem and channel variables. Lower: analysis quantities.
\end{tablenotes}
\end{threeparttable}
\end{table}

\section{Assumption Map}
\label{app:assumption_map}

Table~\ref{tab:assumption_map} records which assumptions each result uses.
The results the experiments test sit on the correlated-Bernoulli voting layer and are computed from observable correctness alone.
The latent channel model \eqref{eq:channel} enters only the latent-rank part of Theorem~\ref{thm:keff}, and the GLS results rest on the separate linear embedding model of \S\ref{sec:gls}.

\begin{table*}[!t]
\caption{Assumption dependencies of each result. ``Observable'' marks results computed from output-level correctness alone.}
\label{tab:assumption_map}
\centering
\vspace{5pt}
\footnotesize
\begin{tabular}{p{4.6cm}p{10.4cm}c}
\toprule
Result & Assumptions used & Observable \\
\midrule
$\Keff^{\text{vote}}$, Eq.~\eqref{eq:keff_vote} & Equicorrelated correctness indicators (design effect~\cite{kish1965}) & \cmark \\
Theorem~\ref{thm:mv} & i.i.d.\ Bernoulli correctness (independence reference) & \cmark \\
Beta-binomial, Eq.~\eqref{eq:mv_corr} & $\Theta \sim \text{Beta}(\alpha,\beta)$, $Y_k \mid \Theta$ i.i.d.\ Bernoulli$(\Theta)$~\cite{skellam1948} & \cmark \\
Adaptive-K, Eq.~\eqref{eq:adaptive_k} & $\Keff^{\text{vote}}$ and a marginal-gain threshold $\varepsilon$ & \cmark \\
Theorem~\ref{thm:keff}, latent rank & Eq.~\eqref{eq:channel} with Assumption~\ref{asm:exchangeable} and the equicorrelated matrix \eqref{eq:corr_matrix} & \xmark \\
Proposition~\ref{prop:bridge} & Gaussian decision surrogate (Assumption~\ref{asm:gaussian_decision}) & \xmark \\
Theorem~\ref{thm:gls} & Linear embedding model $\mathbf{e}_k = g_k \mathbf{s} + \boldsymbol{\varepsilon}_k$~\cite{aitken1935} & \xmark \\
Corollaries~\ref{cor:special}--\ref{cor:gain} & As Theorem~\ref{thm:gls} with $\mathbf{g} = \mathbf{1}$ (and symmetric $\boldsymbol{\Sigma}$ for Corollary~\ref{cor:special}) & \xmark \\
GLS-R, Eq.~\eqref{eq:gls_empirical} & Access to path embeddings & \xmark \\
\bottomrule
\end{tabular}
\end{table*}

\section{Positioning: Analytical Coverage of Prior Work}
\label{app:comparison}

\begin{table*}[!t]
\caption{Analytical coverage of multi-path reasoning studies. Existing methods focus on aggregation strategies, online stopping, or empirical scaling laws; this work provides an upstream correlation-based diagnostic framework. \cmark\ = yes, \xmark\ = no, \pmark\ = partial. Superscripts indicate the predictive mechanism: $^{\mathrm{online}}$ stops sampling per query based on observed agreement/quality; $^{\mathrm{difficulty}}$ predicts $K$ from a query-difficulty mixture model fit to scaling curves; $^{\mathrm{correlation}}$ predicts $K^*$ from a closed-form ceiling derived from inter-path correctness correlation. Compound-inference scaling~\cite{chen2024morellmcalls} fits query-difficulty heterogeneity as a mixture over scaling curves. Our $\hat{c}$ equals the corrected between-instance variance of per-instance accuracy divided by $\bar{p}(1-\bar{p})$ (\S\ref{sec:path_corr}), so it summarizes the same heterogeneity in one vote-level statistic without decomposing its sources.}
\label{tab:comparison}
\centering
\vspace{5pt}
\footnotesize
\begin{tabular}{lccccc}
\toprule
\textbf{Study} & \shortstack{Formal\\saturation\\bound} & \shortstack{Aggregation\\optimality\\analysis} & \shortstack{Predicts\\$K^*$} & \shortstack{Cross-task\\characterization} & \shortstack{Answer-space\\analysis} \\
\midrule
SC~\cite{wang2023selfconsistency}        & \xmark & \xmark & \xmark & \xmark & \xmark \\
CISC~\cite{cisc2025}    & \xmark & \xmark & \xmark & \xmark & \xmark \\
Entropy Voting~\cite{entropyvoting2025}          & \xmark & \xmark & \xmark & \xmark & \xmark \\
Optimal Agg.~\cite{optimalagg2025}     & \xmark & \pmark & \xmark & \xmark & \xmark \\
SoftCoT~\cite{softcot2025,softthinking2025} & \xmark & \xmark & \xmark & \xmark & \xmark \\
Test-time scaling~\cite{snell2024scaling} & \xmark & \xmark & \pmark & \xmark & \xmark \\
Adaptive-Consistency~\cite{aggarwal2023adaptive} & \xmark & \xmark & \pmark$^{\mathrm{online}}$ & \xmark & \xmark \\
ESC~\cite{li2024esc}                              & \xmark & \xmark & \pmark$^{\mathrm{online}}$ & \xmark & \xmark \\
RASC~\cite{wan2025rasc}                           & \xmark & \pmark & \pmark$^{\mathrm{online}}$ & \xmark & \xmark \\
Compound-inf.\ scaling~\cite{chen2024morellmcalls} & \pmark & \xmark & \cmark$^{\mathrm{difficulty}}$ & \pmark & \xmark \\
LLM Monkeys~\cite{brown2024monkeys}               & \xmark & \xmark & \xmark & \pmark & \xmark \\
\midrule
\textbf{This work}      & \cmark & \cmark & \cmark$^{\mathrm{correlation}}$ & \cmark & \cmark \\
\bottomrule
\end{tabular}%
\end{table*}

\section{Design-Effect Derivation}
\label{app:design_effect}

The vote-level effective sample size in \eqref{eq:keff_vote} follows from the standard design-effect formula for correlated binary variables.
For exchangeable $Y_1, \ldots, Y_K$ with $\mathbb{E}[Y_k]=p$ and $\Corr(Y_j, Y_k)=c$ for $j \ne k$, each variance is $p(1{-}p)$ and each of the $K(K{-}1)$ off-diagonal covariances is $c\,p(1{-}p)$, so
\begin{multline}
\Var(\bar{Y}) = \frac{1}{K^2}\Big[K p(1{-}p) + K(K{-}1)\,c\,p(1{-}p)\Big] \\ = p(1{-}p)\,\frac{1 + (K{-}1)c}{K},
\end{multline}
so the effective sample size relative to $K$ i.i.d.\ draws is $\Keff^{\text{vote}} = K / (1 + (K{-}1)c)$~\cite{kish1965}, the number of independent draws whose mean has the same variance $p(1{-}p)/\Keff^{\text{vote}}$.
Eq.~\eqref{eq:corr_c} is the sample analog of $\Corr(Y_j, Y_k)$ pooled over instances, and averaging it over path pairs gives the $\hat{c}$ that enters this formula.
Monotonicity (Proposition~\ref{prop:bridge}) guarantees qualitative consistency (higher $\rho_d$ implies lower ceiling), not numerical equivalence between $c$ and $\rho^2$; the two quantities enter different-level formulas.
The quantitative accuracy of $\hat{c}$ as a predictor of observed MV accuracy is validated empirically in \S\ref{sec:experiments} (Fig.~\ref{fig:bb_calibration}).

\section{Binary Collapse Details and Scope}
\label{app:binary_collapse}

The formal system model in \S\ref{sec:model} assumes a finite answer space $\mathcal{A}$ with nearest-neighbor decoding, which directly applies to closed-form (numeric) and multiple-choice tasks.
For open-form QA tasks, $Y_k = \mathbf{1}\{a_k = x\}$ is computed by alias exact match for TriviaQA and DROP and by token-level F1 $\ge 0.5$ for HotpotQA.
Under this reduction, all theoretical results that depend on binary correctness (Theorem~\ref{thm:mv}, the design-effect formula \eqref{eq:keff_vote}, and the beta-binomial model \eqref{eq:mv_corr}) apply regardless of the original output format.
The GLS analysis (Theorem~\ref{thm:gls}) operates on latent embeddings and does not depend on binary collapse; see \S\ref{sec:gls}.

\section{Kappa vs.~Correctness Correlation}
\label{app:kappa_comparison}

An alternative estimator of path agreement is Cohen's kappa $\kappa = (p_{\text{agree}} - p_{\text{chance}}) / (1 - p_{\text{chance}})$, which measures answer-identity agreement corrected for chance.
Since $\kappa$ conflates ``both correct'' and ``both wrong with the same wrong answer,'' it differs from $c$ in general.
We use $\hat{c}$ in the saturation and Adaptive-K analyses as it is the quantity that enters the effective sample size formula \eqref{eq:keff_vote}.

\section{Proof of Theorem~\ref{thm:keff} (Effective Rank under Equicorrelation)}
\label{app:keff_proof}

The eigenvalues of the equally-correlated matrix $\mathbf{R}$ with diagonal $\sigma_h^2$ and off-diagonal $\sigma_h^2 \rho$ are:
\begin{align}
\lambda_1 &= \sigma_h^2 (1 + (K{-}1)\rho), \quad \text{[multiplicity 1]} \\
\lambda_j &= \sigma_h^2 (1 - \rho), \quad j = 2, \ldots, K.
\end{align}
The trace and squared trace are:
\begin{align}
\tr \mathbf{R} &= K \sigma_h^2, \\
\tr(\mathbf{R}^2) &= \sigma_h^4 \bigl[ (1 + (K{-}1)\rho)^2 + (K{-}1)(1{-}\rho)^2 \bigr].
\end{align}
Expanding the denominator:
\begin{align}
&(1 + (K{-}1)\rho)^2 + (K{-}1)(1{-}\rho)^2 \notag \\
&= 1 + 2(K{-}1)\rho + (K{-}1)^2 \rho^2 \notag \\
&\quad + (K{-}1) - 2(K{-}1)\rho + (K{-}1)\rho^2 \notag \\
&= K + K(K{-}1)\rho^2.
\end{align}
Therefore $\Keff^{\text{rank}} = K^2 \sigma_h^4 / [\sigma_h^4 (K + K(K{-}1)\rho^2)] = K / (1 + (K{-}1)\rho^2)$.
Taking $K \to \infty$ yields $\Keff^{\text{rank}} \to 1/\rho^2$.

\section{Proof of Theorem~\ref{thm:gls} and Special Cases}
\label{app:gls_proof}

\noindent\textbf{Proof of Theorem~\ref{thm:gls}.}
Form the Lagrangian $\mathcal{L}(\mathbf{w}, \nu) = \mathbf{w}^\top \boldsymbol{\Sigma} \mathbf{w} + \nu (\mathbf{g}^\top \mathbf{w} - 1)$.
Stationarity gives $2\boldsymbol{\Sigma} \mathbf{w} + \nu \mathbf{g} = \mathbf{0}$, hence $\mathbf{w} = -(\nu/2) \boldsymbol{\Sigma}^{-1} \mathbf{g}$.
Substituting into the constraint $\mathbf{g}^\top \mathbf{w} = 1$ yields the stated formula $\mathbf{w}_* = \boldsymbol{\Sigma}^{-1} \mathbf{g} / (\mathbf{g}^\top \boldsymbol{\Sigma}^{-1} \mathbf{g})$.

\smallskip\noindent\textbf{Additional special cases of Corollary~\ref{cor:special}.}
Beyond the equicorrelated symmetric case stated in the main text, the GLS combiner recovers:
\begin{enumerate}
\item \textbf{White noise} ($\boldsymbol{\Sigma} = \sigma^2 \mathbf{I}$): $\mathbf{w}_* \propto \mathbf{g}$, i.e., quality-weighted combining (the MRC principle).
\item \textbf{Equal gains} ($\mathbf{g} = \mathbf{1}$): $\mathbf{w}_* = \boldsymbol{\Sigma}^{-1} \mathbf{1} / (\mathbf{1}^\top \boldsymbol{\Sigma}^{-1} \mathbf{1})$, i.e., covariance-aware averaging.
\end{enumerate}

\section{Heterogeneous Slot Regime: Weighted Voting}
\label{app:hetero_slot}

This appendix collects the empirical follow-ups to Corollary~\ref{cor:gain} for prompt-template (PT) data, where heterogeneous per-slot accuracy breaks the symmetric exchangeable regime.

\smallskip\noindent\textbf{When does uniform voting fail?}
We compute accuracy-weighted voting, with each slot weighted by its accuracy on the evaluation set, across all 60 PT cells and measure the gain over uniform MV as a function of the slot-accuracy coefficient of variation (CV).
The correlation between CV and weighted-MV gain is $r{=}0.85$ ($p{<}10^{-4}$). In the 7 cells with heterogeneous slot accuracies (CV $> 0.225$), weighted voting gains ${+}4.4$ pp on average over uniform MV, and in the 53 near-homogeneous cells (CV $\le 0.225$) the gain is ${+}1.1$ pp.
The gain is larger in the low-accuracy regime. Cells with MV accuracy below 30\% gain ${+}2.2$ pp on average (mean CV $0.19$), and cells at or above 30\% gain ${+}1.2$ pp (mean CV $0.07$). Of the low-accuracy cells, 31\% have CV $> 0.225$, against 5\% of the high-accuracy cells.

\smallskip\noindent\textbf{Weighting in the SC regime.}
Slot weights estimated on the evaluation set give an optimistic, in-sample estimate of what slot weighting can achieve, and deployment would require held-out calibration.
Under standard SC every slot CV is at most $0.10$, and the same in-sample accuracy weighting changes MV accuracy by between ${-}0.7$ and ${+}1.8$ pp across the 60 cells.
We read this as an empirical observation that slot-level accuracy weighting has little room to improve on uniform MV in the near-exchangeable SC regime. It is not a bound on other weighting schemes, and confidence-based methods such as CISC~\cite{cisc2025} use per-path signals that slot-level weights do not capture.

\section{Scope and Boundaries}
\label{app:scope_boundaries}

\smallskip\noindent\textbf{What the framework enables.}
The saturation formula \eqref{eq:keff_vote} uses the observable $\hat{c}$ directly at the vote level; Proposition~\ref{prop:bridge} justifies treating $\hat{c}$ as a monotone proxy for the unobservable latent dependence.
The GLS derivation establishes that, with equal path gains, uniform weighting is optimal among linear combiners of latent embeddings when the path covariance has equal row sums, as under exchangeability (Theorem~\ref{thm:gls}, Corollary~\ref{cor:special}). Majority vote on decoded answers is a different, discrete rule, so this result motivates uniform voting without establishing its optimality (Remark~\ref{rem:gls_to_mv}).
The channel model motivates path correlation as the diagnostic variable, and prompt-template perturbation moves it in all six task domains (\S\ref{sec:cross_bench}).

\smallskip\noindent\textbf{Practical applicability.}
The framework is designed as an offline diagnostic tool for multi-path reasoning pipelines.
Applying it to a new model-task pair requires (1) a calibration run of $K{=}4$ SC paths on $n \approx 50$--$100$ instances to estimate $\hat{c}$ (Appendix~\ref{app:pilot_sensitivity} confirms $n{=}50$ yields $\hat{c}$ with CV $\le 14.2\%$ and $K^*$ std $\le 2.0$), and (2) evaluating the closed-form expressions for $\Keff^{\text{vote}}$ and $K^*$.
The calibration run costs $4n$ path generations once per model-task pair, which is $12.5\%$ of a $K{=}32$ budget on the calibration instances and is amortized over all later queries, which need no additional inference.
With $n{=}100$, the one-time pilot costs $400$ paths, and serving $N$ queries at $K^*$ costs $400 + K^* N$ paths against $32N$ at fixed $K{=}32$.
On the five cells of Table~\ref{tab:adaptive_k} the pilot is repaid after $15$--$19$ queries, and at $N{=}1000$ the total cost including the pilot is $13.8$--$32.5\%$ of the fixed-$K{=}32$ budget.
Calibration uses gold labels on the $n$ pilot instances only, and deployment is label-free.
The estimate belongs to one (model, prompt, task) configuration: a model update, a prompt change or a distribution shift defines a new configuration and triggers recalibration, and transfer of an operating point without recalibration is outside the claimed regime.
For API-based deployments where hidden-state access is unavailable, the framework remains applicable: $\hat{c}$ is estimated from output-level correctness patterns alone, and the Adaptive-K rule operates entirely on observed vote statistics.

\smallskip\noindent\textbf{Equicorrelated approximation quality.}
The theoretical results assume Assumption~\ref{asm:exchangeable} with a common correlation $\rho$.
Prompt-template diversity breaks this symmetry: different templates produce heterogeneous slot accuracy.
To quantify the approximation, we compute the variance-based $\Keff^{\text{block}} = \bar{p}(1{-}\bar{p}) / \Var(S_K/K)$ from the full $K{\times}K$ correlation matrix across all 60 PT cells and compare to the equicorrelated $\Keff^{\text{vote}}$.
The mean ratio $\Keff^{\text{block}} / \Keff^{\text{vote}}$ is $1.01$ (std $0.03$, maximum $1.15$ on Qwen-7B GSM8K), so the equicorrelated formula is accurate to within a few percent on average and slightly conservative.
The binary-collapsed correctness model loses information about the wrong-answer distribution in multiclass settings, which $\Keff^{\text{vote}}$ does not capture.

\section{Statistical Precision of $\Delta\rho$ Estimates}
\label{app:statistical_precision}

Bootstrap confidence intervals (10{,}000 paired resamples of the pooled (seed, instance) rows per cell, with the same rows drawn for both arms) at the pooled sample size ($n{=}500$ for GSM8K, MATH, TriviaQA, DROP, CruxEval and MBPP, $n{=}300$ for the other six benchmarks) give per-cell 95\% CI widths with median $16$ pp (IQR $12$--$22$ pp).
All ten QA cells exclude zero (TriviaQA across all five models lies within $[-47\%, -3\%]$).
For MATH, the Qwen-7B and Qwen-32B cells include zero (the Qwen-32B interval is $[-7\%, +3\%]$), consistent with their small effects.
Across the 60 cells, 38 have CIs that exclude zero at the 95\% level, 37 of them below zero, and the remaining 22 intervals include zero and do not resolve the sign of the effect.
Per-cell $\Delta\rho$ estimates span $[-52\%, +90\%]$, and the positive extreme is the chance-level Qwen-0.5B WinoGrande cell discussed in \S\ref{sec:cross_bench}.
No cell has $\bar{p} < 2\%$, so none is excluded.

\section{Per-Cell Prompt-Template Results}
\label{app:cell_matrix}

Table~\ref{tab:full_matrix} reports the per-benchmark means across the five models, and Table~\ref{tab:cell_matrix} lists all 60 model-benchmark cells behind them and Fig.~\ref{fig:drho_cells}.

\begin{table*}[!t]
\caption{Prompt-template diversity across 12 benchmarks ($K{=}8$, $100$ instances per seed, pooled over 5 seeds for GSM8K, MATH, TriviaQA, DROP, CruxEval and MBPP and over 3 seeds for the other six benchmarks). Five models evaluated: Qwen2.5-0.5B/7B/32B, Llama-3.1-8B, Mistral-7B. The \emph{Answer space} column gives the effective evaluated output space (post-extraction value compared against gold), used as the relevant axis for the correlation analysis. $\Delta\rho$ is the \emph{relative} change in pairwise correlation, $(\hat{\rho}_{\text{PT}}-\hat{\rho}_{\text{SC}})/\hat{\rho}_{\text{SC}}$, expressed as a percentage (so $-20.3$ on TriviaQA means correlation drops by $20.3\%$ of its SC value, not by 20.3 percentage points). No cell falls below the $2\%$ base-accuracy cut, so $n{=}5$ models for every benchmark, and each row reports the mean across them. The WinoGrande mean is driven by the chance-level Qwen-0.5B cell (median ${-}3.2\%$). Per-cell values are in Table~\ref{tab:cell_matrix}.
}
\label{tab:full_matrix}
\centering
\vspace{5pt}
\footnotesize
\begin{tabular}{lllcrrrr}
\toprule
Domain & Benchmark & Answer space & $n$ & Mean $\Delta\rho$ (\%) $\downarrow$ & Mean $\Delta\Keff^{\text{vote}}$ $\uparrow$ & Mean $\hat{\rho}_{\text{SC}}$ & Mean $\hat{\rho}_{\text{PT}}$ \\
\midrule
\multirow{2}{*}{QA}
  & TriviaQA & open text & 5 & $-20.3$ & $+0.3$ & 0.70 & 0.57 \\
  & HotpotQA & open text & 5 & $-25.0$ & $+0.4$ & 0.71 & 0.53 \\
\midrule
\multirow{2}{*}{Science/MC}
  & ARC-C & bounded (4) & 5 & $-6.7$ & $+0.1$ & 0.54 & 0.51 \\
  & MMLU & bounded (4) & 5 & $-20.5$ & $+0.4$ & 0.41 & 0.35 \\
\midrule
\multirow{2}{*}{Commonsense}
  & HellaSwag & bounded (4) & 5 & $-9.0$ & $+0.2$ & 0.47 & 0.44 \\
  & WinoGrande & binary (2) & 5 & $+6.9$ & $+0.0$ & 0.41 & 0.42 \\
\midrule
\multirow{2}{*}{NLU}
  & DROP & open text & 5 & $-22.6$ & $+0.4$ & 0.59 & 0.45 \\
  & BoolQ & binary (y/n) & 5 & $-16.9$ & $+0.2$ & 0.78 & 0.65 \\
\midrule
\multirow{2}{*}{Code}
  & MBPP & code pass/fail & 5 & $-19.6$ & $+0.3$ & 0.62 & 0.49 \\
  & CruxEval & code pass/fail & 5 & $-19.1$ & $+0.3$ & 0.53 & 0.42 \\
\midrule
\multirow{2}{*}{Math}
  & GSM8K & numeric & 5 & $-33.6$ & $+0.7$ & 0.48 & 0.31 \\
  & MATH & numeric & 5 & $-7.9$ & $+0.2$ & 0.51 & 0.47 \\
\bottomrule
\end{tabular}%
\end{table*}

\begin{table*}[!t]
\caption{All 60 model-benchmark cells ($K{=}8$). $n$ is the number of pooled instances, $\bar{p}$ the mean per-path accuracy (\%), $\hat{\rho}$ the mean pairwise correlation of slot correctness, $\Delta\rho$ the relative change (\%) with its paired bootstrap 95\% CI, and MV the majority-vote accuracy (\%) with ties scored as $0.5$.}
\label{tab:cell_matrix}
\centering
\vspace{5pt}
\scriptsize
\begin{tabular}{llcrrrrrcrr}
\toprule
Benchmark & Model & $n$ & $\bar{p}_{\text{SC}}$ & $\bar{p}_{\text{PT}}$ & $\hat{\rho}_{\text{SC}}$ & $\hat{\rho}_{\text{PT}}$ & $\Delta\rho$ & 95\% CI & $\mathrm{MV}_{\text{SC}}$ & $\mathrm{MV}_{\text{PT}}$ \\
\midrule
TriviaQA & Qwen-0.5B & 500 & 37.3 & 34.8 & 0.65 & 0.55 & $-14.8$ & [$-19.4$, $-10.0$] & 36.8 & 34.6 \\
 & Qwen-7B & 500 & 62.0 & 61.1 & 0.87 & 0.81 & $-6.2$ & [$-9.3$, $-3.1$] & 62.4 & 62.9 \\
 & Qwen-32B & 500 & 67.1 & 63.8 & 0.84 & 0.76 & $-9.0$ & [$-12.5$, $-5.8$] & 67.4 & 63.4 \\
 & Llama-8B & 500 & 30.1 & 30.9 & 0.52 & 0.36 & $-30.6$ & [$-36.9$, $-24.0$] & 29.1 & 28.7 \\
 & Mistral-7B & 500 & 28.6 & 25.9 & 0.61 & 0.36 & $-41.1$ & [$-46.7$, $-35.4$] & 26.9 & 23.3 \\
\midrule
HotpotQA & Qwen-0.5B & 300 & 27.7 & 24.3 & 0.68 & 0.62 & $-8.8$ & [$-15.8$, $-1.7$] & 27.2 & 21.0 \\
 & Qwen-7B & 300 & 66.5 & 69.5 & 0.86 & 0.60 & $-30.4$ & [$-36.7$, $-24.4$] & 66.7 & 70.8 \\
 & Qwen-32B & 300 & 80.4 & 80.3 & 0.79 & 0.64 & $-18.9$ & [$-25.6$, $-13.1$] & 80.2 & 83.5 \\
 & Llama-8B & 300 & 50.9 & 50.6 & 0.59 & 0.37 & $-36.5$ & [$-43.8$, $-28.7$] & 51.0 & 54.8 \\
 & Mistral-7B & 300 & 30.2 & 31.3 & 0.60 & 0.42 & $-30.2$ & [$-36.8$, $-23.4$] & 30.0 & 29.5 \\
\midrule
ARC-C & Qwen-0.5B & 300 & 33.1 & 32.8 & 0.22 & 0.18 & $-20.5$ & [$-39.8$, $+4.2$] & 24.8 & 22.5 \\
 & Qwen-7B & 300 & 91.0 & 90.4 & 0.70 & 0.65 & $-7.3$ & [$-17.9$, $+4.7$] & 92.5 & 92.0 \\
 & Qwen-32B & 300 & 92.6 & 92.9 & 0.76 & 0.71 & $-7.2$ & [$-15.4$, $+0.7$] & 93.7 & 94.0 \\
 & Llama-8B & 300 & 76.2 & 76.8 & 0.51 & 0.52 & $+1.8$ & [$-8.2$, $+13.0$] & 77.8 & 79.3 \\
 & Mistral-7B & 300 & 80.4 & 74.9 & 0.52 & 0.52 & $-0.2$ & [$-11.6$, $+13.7$] & 83.0 & 76.5 \\
\midrule
MMLU & Qwen-0.5B & 300 & 25.3 & 24.0 & 0.07 & 0.04 & $-39.8$ & [$-77.0$, $+25.4$] & 11.2 & 8.5 \\
 & Qwen-7B & 300 & 70.8 & 69.4 & 0.63 & 0.55 & $-12.4$ & [$-19.4$, $-4.9$] & 72.7 & 71.7 \\
 & Qwen-32B & 300 & 85.2 & 84.0 & 0.65 & 0.60 & $-8.1$ & [$-16.1$, $+0.4$] & 86.3 & 85.2 \\
 & Llama-8B & 300 & 41.7 & 40.8 & 0.38 & 0.28 & $-25.3$ & [$-35.9$, $-13.4$] & 41.3 & 39.3 \\
 & Mistral-7B & 300 & 35.6 & 33.4 & 0.33 & 0.27 & $-16.7$ & [$-31.5$, $+0.4$] & 32.7 & 27.7 \\
\midrule
HellaSwag & Qwen-0.5B & 300 & 32.2 & 32.5 & 0.12 & 0.09 & $-23.4$ & [$-49.5$, $+16.0$] & 22.2 & 23.5 \\
 & Qwen-7B & 300 & 76.5 & 74.0 & 0.62 & 0.59 & $-4.8$ & [$-12.6$, $+3.7$] & 79.2 & 75.3 \\
 & Qwen-32B & 300 & 82.0 & 81.4 & 0.67 & 0.64 & $-4.9$ & [$-11.9$, $+2.0$] & 82.8 & 82.7 \\
 & Llama-8B & 300 & 59.0 & 59.1 & 0.46 & 0.38 & $-16.9$ & [$-25.7$, $-7.6$] & 61.8 & 61.8 \\
 & Mistral-7B & 300 & 61.5 & 60.8 & 0.48 & 0.50 & $+5.3$ & [$-5.5$, $+17.7$] & 63.8 & 61.7 \\
\midrule
WinoGrande & Qwen-0.5B & 300 & 49.2 & 44.8 & 0.23 & 0.43 & $+89.6$ & [$+59.1$, $+129.5$] & 50.5 & 45.7 \\
 & Qwen-7B & 300 & 76.1 & 75.8 & 0.51 & 0.49 & $-3.2$ & [$-12.4$, $+6.9$] & 79.2 & 78.7 \\
 & Qwen-32B & 300 & 84.6 & 83.4 & 0.66 & 0.65 & $-1.4$ & [$-9.3$, $+7.0$] & 86.2 & 84.7 \\
 & Llama-8B & 300 & 62.8 & 62.3 & 0.26 & 0.15 & $-42.5$ & [$-56.4$, $-27.1$] & 66.0 & 68.7 \\
 & Mistral-7B & 300 & 63.6 & 60.1 & 0.40 & 0.37 & $-8.1$ & [$-21.8$, $+8.0$] & 66.8 & 60.8 \\
\midrule
DROP & Qwen-0.5B & 500 & 15.6 & 15.2 & 0.48 & 0.40 & $-15.4$ & [$-24.3$, $-5.6$] & 14.0 & 11.8 \\
 & Qwen-7B & 500 & 56.0 & 64.8 & 0.76 & 0.43 & $-43.0$ & [$-48.5$, $-37.6$] & 54.6 & 69.2 \\
 & Qwen-32B & 500 & 78.1 & 79.5 & 0.63 & 0.62 & $-1.2$ & [$-7.0$, $+5.1$] & 81.1 & 82.3 \\
 & Llama-8B & 500 & 35.5 & 34.2 & 0.43 & 0.31 & $-28.6$ & [$-36.6$, $-20.1$] & 34.7 & 30.9 \\
 & Mistral-7B & 500 & 12.0 & 10.4 & 0.64 & 0.48 & $-24.8$ & [$-33.9$, $-16.5$] & 10.8 & 7.1 \\
\midrule
BoolQ & Qwen-0.5B & 300 & 60.6 & 52.5 & 0.49 & 0.44 & $-10.9$ & [$-19.4$, $-1.5$] & 64.2 & 53.2 \\
 & Qwen-7B & 300 & 81.7 & 83.0 & 0.80 & 0.73 & $-8.7$ & [$-16.5$, $-0.6$] & 81.0 & 83.5 \\
 & Qwen-32B & 300 & 85.0 & 85.9 & 0.82 & 0.73 & $-10.7$ & [$-17.0$, $-5.0$] & 84.7 & 87.2 \\
 & Llama-8B & 300 & 80.0 & 78.5 & 0.80 & 0.47 & $-41.6$ & [$-48.3$, $-34.9$] & 79.7 & 80.7 \\
 & Mistral-7B & 300 & 83.8 & 81.9 & 0.98 & 0.86 & $-12.6$ & [$-17.1$, $-8.6$] & 83.8 & 81.8 \\
\midrule
MBPP & Qwen-0.5B & 500 & 31.9 & 32.1 & 0.41 & 0.40 & $-2.4$ & [$-11.7$, $+7.9$] & 31.5 & 29.8 \\
 & Qwen-7B & 500 & 80.8 & 77.5 & 0.82 & 0.59 & $-27.3$ & [$-32.7$, $-22.0$] & 81.2 & 80.4 \\
 & Qwen-32B & 500 & 89.7 & 87.7 & 0.74 & 0.59 & $-19.3$ & [$-27.3$, $-11.2$] & 91.4 & 90.4 \\
 & Llama-8B & 500 & 74.1 & 69.3 & 0.71 & 0.54 & $-24.3$ & [$-29.2$, $-19.5$] & 75.3 & 73.1 \\
 & Mistral-7B & 500 & 37.1 & 35.7 & 0.45 & 0.34 & $-24.6$ & [$-32.2$, $-16.4$] & 39.1 & 36.7 \\
\midrule
CruxEval & Qwen-0.5B & 500 & 12.2 & 10.1 & 0.62 & 0.54 & $-12.5$ & [$-21.7$, $-3.4$] & 10.2 & 8.7 \\
 & Qwen-7B & 500 & 58.9 & 57.6 & 0.56 & 0.46 & $-16.5$ & [$-21.7$, $-11.0$] & 59.0 & 61.4 \\
 & Qwen-32B & 500 & 83.0 & 80.5 & 0.57 & 0.49 & $-14.6$ & [$-21.3$, $-7.9$] & 84.1 & 85.1 \\
 & Llama-8B & 500 & 16.2 & 29.8 & 0.24 & 0.21 & $-11.0$ & [$-27.9$, $+11.7$] & 9.6 & 25.3 \\
 & Mistral-7B & 500 & 34.1 & 32.7 & 0.66 & 0.39 & $-41.0$ & [$-46.3$, $-35.3$] & 34.5 & 32.9 \\
\midrule
GSM8K & Qwen-0.5B & 500 & 33.3 & 31.1 & 0.37 & 0.37 & $+1.3$ & [$-9.4$, $+13.1$] & 31.0 & 27.3 \\
 & Qwen-7B & 500 & 79.2 & 70.7 & 0.59 & 0.31 & $-48.4$ & [$-54.6$, $-42.0$] & 82.0 & 80.9 \\
 & Qwen-32B & 500 & 80.5 & 73.3 & 0.45 & 0.22 & $-51.9$ & [$-59.0$, $-44.5$] & 83.9 & 84.1 \\
 & Llama-8B & 500 & 75.1 & 69.0 & 0.52 & 0.36 & $-30.1$ & [$-36.6$, $-23.3$] & 78.1 & 73.6 \\
 & Mistral-7B & 500 & 42.8 & 38.0 & 0.46 & 0.28 & $-39.0$ & [$-46.0$, $-31.8$] & 42.8 & 37.4 \\
\midrule
MATH & Qwen-0.5B & 500 & 21.1 & 24.0 & 0.33 & 0.37 & $+13.3$ & [$-1.0$, $+29.7$] & 16.0 & 19.9 \\
 & Qwen-7B & 500 & 75.8 & 74.6 & 0.63 & 0.61 & $-2.9$ & [$-8.2$, $+2.6$] & 75.9 & 75.8 \\
 & Qwen-32B & 500 & 81.1 & 80.8 & 0.67 & 0.66 & $-2.2$ & [$-7.4$, $+3.2$] & 83.8 & 84.0 \\
 & Llama-8B & 500 & 45.1 & 43.4 & 0.57 & 0.49 & $-14.3$ & [$-20.4$, $-8.1$] & 45.9 & 43.3 \\
 & Mistral-7B & 500 & 14.1 & 11.9 & 0.35 & 0.23 & $-33.3$ & [$-45.3$, $-19.5$] & 10.3 & 6.5 \\
\bottomrule
\end{tabular}
\end{table*}

\section{Prompt-Template MV Accuracy Deltas}
\label{app:pt_acc_delta}

Table~\ref{tab:pt_acc_delta} reports per-domain MV accuracy under SC and prompt-template (PT), and the difference $\Delta\mathrm{Acc} = \mathrm{MV}_{\text{PT}} - \mathrm{MV}_{\text{SC}}$, averaged over the model-benchmark cells of each domain ($K{=}8$, seeds pooled as in Table~\ref{tab:full_matrix}). $\Delta\mathrm{Acc}$ is small (within $\pm 2$ pp on average) and mixed in sign across the six domains, with within-domain standard deviations comparable to or larger than the means. This pattern is consistent with the $\Keff^{\text{vote}}$ analysis. PT enlarges the saturation ceiling $1/c$ but does not by itself raise base per-path accuracy $\bar{p}$, so the operational MV accuracy remains near the SC level. The diagnostic value of PT in this paper is therefore the \emph{structure} it reveals (Section~\ref{sec:cross_bench}, Table~\ref{tab:full_matrix}), not a uniform accuracy improvement over SC.

\begin{table}[h]
\caption{Per-domain MV accuracy under SC and PT ($K{=}8$, seeds pooled as in Table~\ref{tab:full_matrix}). $n$ = number of model-benchmark cells (none falls below the $2\%$ base-accuracy cut). $\pm$ denotes within-domain standard deviation across cells.}
\label{tab:pt_acc_delta}
\centering
\vspace{5pt}
\small
\begin{tabular}{lcccc}
\toprule
Domain & $n$ & $\mathrm{MV}_{\text{SC}}$ & $\mathrm{MV}_{\text{PT}}$ & $\Delta\mathrm{Acc}$ (pp) \\
\midrule
QA & 10 & 47.8\% & 47.3\% & $-0.5 \pm 3.5$ \\
Sci/MC & 10 & 61.6\% & 59.7\% & $-1.9 \pm 2.4$ \\
Commonsense & 10 & 65.8\% & 64.3\% & $-1.5 \pm 2.7$ \\
NLU & 10 & 58.9\% & 58.8\% & $-0.1 \pm 6.5$ \\
Code & 10 & 51.6\% & 52.4\% & $+0.8 \pm 5.4$ \\
Math & 10 & 55.0\% & 53.3\% & $-1.7 \pm 2.8$ \\
\bottomrule
\end{tabular}
\end{table}

\section{Reasoning Model}
\label{app:reasoning}

We apply the saturation and Adaptive-K protocol to Qwen3.5-9B in thinking mode on GSM8K, with 3 seeds (42, 123, 456) $\times$ 100 instances, $K{=}32$ paths per instance sampled at $\tau{=}0.7$ and top-$p$ $0.95$, and a generation cap of 16{,}384 tokens.
Each $K$ is evaluated on the first $K$ of the 32 paths.
5.4\% of paths reach the cap, and these are scored as incorrect.
The model is close to its accuracy ceiling on this task, with mean per-path accuracy 92.9\%.
The $K{=}4$ pilot gives $\hat{c}{=}0.33$, and at $K{=}32$ the measured $\hat{c}{=}0.38$ gives a ceiling $1/\hat{c}{=}2.67$ with $\Keff^{\text{vote}}{=}2.53$.

Table~\ref{tab:reasoning} reports the held-out beta-binomial test with the protocol of \S\ref{sec:bb_calibration}. The parameters $(\alpha, \beta)$ are fitted from the $K{=}4$ pilot on one half of the instances and used to predict MV@$K$ on the other half.
The beta-binomial error is $0.8$--$1.7$ pp across $K$, whereas the independence model is off by $2.8$--$3.2$ pp.
Adaptive-K selects $K^*{=}14$ and retains 99.8\% of MV@32 (paired bootstrap 95\% CI of MV@$K^*{-}$MV@32: $[-0.5, 0.0]$ pp) at a net cost of 56\% of the fixed $K{=}32$ budget.
Harder benchmarks such as AIME need a generation budget above 16{,}384 tokens and are left to future work.

\begin{table}[h]
\caption{Qwen3.5-9B (thinking) on GSM8K, 3 seeds $\times$ 100 instances. Observed binary MV and held-out absolute prediction errors (pp), averaged over the two instance halves.}
\label{tab:reasoning}
\centering
\vspace{5pt}
\footnotesize
\begin{tabular}{cccc}
\toprule
$K$ & Observed MV & Beta-binomial error & Independence error \\
\midrule
4 & 95.8\% & 0.8 & 2.8 \\
8 & 96.8\% & 1.2 & 3.1 \\
16 & 96.8\% & 1.3 & 3.2 \\
32 & 96.8\% & 1.7 & 3.2 \\
\bottomrule
\end{tabular}
\end{table}

\section{Limitations and Future Work}
\label{app:limitations}

\smallskip\noindent\textbf{Limitations.}
The Adaptive-K rule uses a global $\hat{c}$ estimate and does not adapt to per-instance difficulty variation; incorporating instance-level confidence into $K^*$ is a natural refinement.
The GLS analysis (Theorem~\ref{thm:gls}) requires access to hidden-state embeddings; for API-only deployments the vote-level framework remains applicable with the observable $\hat{c}$ (Proposition~\ref{prop:bridge}), and the Adaptive-K rule operates entirely on output statistics.
Per-cell $\Delta\rho$ estimates at the pooled sample size ($n{=}300$ or $500$) have 95\% CI widths with median $16$ pp (IQR $12$--$22$ pp) (Appendix~\ref{app:statistical_precision}), and the intervals of individual small-effect cells, such as MATH on Qwen-7B and Qwen-32B, include zero and would need larger $n$ to resolve the sign.
Adaptive-K fixes one budget per (model, prompt, task) configuration before deployment, while online-stopping methods~\cite{aggarwal2023adaptive,li2024esc,wan2025rasc} act per query during sampling. A matched-compute comparison with them remains open.

\smallskip\noindent\textbf{Future work.}
Optimal prompt-template design under a decorrelation--accuracy trade-off, and extensions to block-covariance or non-exchangeable regimes, remain open.

\section{Beta-Binomial Correlated Voting Model}
\label{app:betabin}

Under the beta-binomial model $\Theta \sim \text{Beta}(\alpha, \beta)$, $Y_k \mid \Theta \sim \text{Bernoulli}(\Theta)$, the moments are $p = \mathbb{E}[\Theta] = \alpha/(\alpha+\beta)$ and, for $j \ne k$, $\Cov(Y_j, Y_k) = \Var(\Theta) = \alpha\beta/[(\alpha+\beta)^2(\alpha+\beta+1)]$.
Dividing by $\Var(Y_k) = p(1{-}p) = \alpha\beta/(\alpha+\beta)^2$ gives $c = 1/(\alpha+\beta+1)$, so $\alpha+\beta = (1{-}c)/c$, and the method-of-moments estimates follow from $\alpha = p(\alpha+\beta)$ and $\beta = (1{-}p)(\alpha+\beta)$:
\begin{equation}
\alpha = \frac{p(1-c)}{c}, \quad \beta = \frac{(1-p)(1-c)}{c},
\end{equation}
where $p = \mathbb{E}[Y_k]$ and $c = \Corr(Y_i, Y_j) = 1/(\alpha + \beta + 1)$.
The probability mass function of the count $S_K = \sum_k Y_k$ is:
\begin{equation}
P(S_K = j) = \binom{K}{j} \frac{B(j + \alpha, K - j + \beta)}{B(\alpha, \beta)}.
\end{equation}
The correlated majority-vote accuracy is:
\begin{equation}
P_{\text{MV}}^{\text{corr}} = \sum_{j > K/2} P(S_K{=}j).
\end{equation}
For even $K$ with random tie-breaking, add $\tfrac{1}{2}\, P(S_K{=}K/2)$.

\section{Pilot Size Sensitivity}
\label{app:pilot_sensitivity}

Table~\ref{tab:pilot_sensitivity} shows the stability of $\hat{c}$ and $K^*$ estimates as a function of pilot size $n$, computed via 500 bootstrap resamples of instances (with all seed replicates of an instance kept together) from the pooled 5-seed $K{=}4$ SC data on GSM8K.
At $n{=}50$, $\hat{c}$ has a coefficient of variation of at most $14.2\%$ across all three models and $K^*$ standard deviation is $1.5$--$2.0$, sufficient for practical guidance.
At $n{=}25$, $\hat{c}$ CV reaches $20.5$--$21.5\%$ and $K^*$ variance widens (std $3.1$--$6.8$), motivating $n{\ge}50$ as the minimum recommended pilot size.

\begin{table}[h]
\caption{Pilot size sensitivity on GSM8K ($K{=}4$, 500 bootstrap resamples of instances from pooled 5-seed data).}
\label{tab:pilot_sensitivity}
\centering
\vspace{5pt}
\begin{tabular}{llcccc}
\toprule
Model & $n$ & $\hat{c}$ mean & $\hat{c}$ std & $K^*$ mean & $K^*$ std \\
\midrule
\multirow{4}{*}{Qwen-7B}
  & 25 & 0.571 & 0.122 & 7.9 & 6.77 \\
  & 50 & 0.589 & 0.083 & 6.9 & 1.69 \\
  & 100 & 0.596 & 0.055 & 6.6 & 1.02 \\
  & 200 & 0.605 & 0.039 & 6.5 & 0.72 \\
\midrule
\multirow{4}{*}{Llama-8B}
  & 25 & 0.512 & 0.105 & 8.7 & 3.14 \\
  & 50 & 0.522 & 0.068 & 8.1 & 1.53 \\
  & 100 & 0.526 & 0.048 & 8.0 & 1.06 \\
  & 200 & 0.525 & 0.034 & 7.9 & 0.77 \\
\midrule
\multirow{4}{*}{Mistral-7B}
  & 25 & 0.437 & 0.090 & 10.6 & 3.10 \\
  & 50 & 0.438 & 0.062 & 10.3 & 1.98 \\
  & 100 & 0.447 & 0.044 & 9.9 & 1.30 \\
  & 200 & 0.446 & 0.031 & 9.9 & 0.92 \\
\bottomrule
\end{tabular}
\end{table}

\section{Proof of Corollary~\ref{cor:gain}}
\label{app:gain_proof}

Let $\mathbf{w}_u = (1/K)\mathbf{1}$.
Since $\mathbf{w}_*$ minimizes $\mathbf{w}^\top \boldsymbol{\Sigma} \mathbf{w}$ subject to $\mathbf{g}^\top \mathbf{w} = 1$, and $\mathbf{w}_u$ is feasible when $\mathbf{g} = \mathbf{1}$, the optimal value cannot exceed the feasible value: $\MSE(\mathbf{w}_*) \le \MSE(\mathbf{w}_u)$.
Equality holds when $\mathbf{w}_u$ itself satisfies the KKT conditions, \ie when $\boldsymbol{\Sigma} \mathbf{w}_u \propto \mathbf{g}$.
For $\mathbf{g} = \mathbf{1}$, this requires $\boldsymbol{\Sigma} \mathbf{1} \propto \mathbf{1}$, \ie all row sums of $\boldsymbol{\Sigma}$ are equal.

\section{Adaptive-K Algorithm}
\label{app:algorithm}

Algorithm~\ref{alg:adaptive_k} states the procedure behind \S\ref{sec:adaptive_k}.
The pilot stage is the only stage that uses gold labels, and only on the $n$ pilot instances.
The deployment stage is label-free.
The pilot's compute is $4n$ sampled paths once per (model, prompt, task) configuration, which the net-cost column of Table~\ref{tab:adaptive_k} already charges.

\begin{algorithm}[h]
\caption{Adaptive-K path-budget selection}
\label{alg:adaptive_k}
\begin{algorithmic}[1]
\REQUIRE model $M$, task instances $\mathcal{D}$, pilot size $n$ (default $50$--$100$), threshold $\varepsilon > 0$ (default $0.025$), maximum budget $K_{\max}$ (default $32$)
\ENSURE operating point $K^*$ and final answers
\STATE \textbf{Pilot:} for each of $n$ pilot instances, sample $K_0{=}4$ paths from $M$
\STATE Extract discrete answers and score them against gold to obtain $Y_i^{(k)}$
\STATE Estimate $\bar{p}$ and $\hat{c}$ from the $Y_i^{(k)}$ by \eqref{eq:corr_c}
\IF{$\bar{p} \in \{0, 1\}$}
\STATE $K^* \leftarrow K_{\max}$ \COMMENT{degenerate pilot}
\ELSE
\STATE $\hat{c} \leftarrow \mathrm{clip}(\hat{c}, 0.05, 0.99)$; $K^* \leftarrow \min\big(\max(\lceil (\sqrt{(1{-}\hat{c})/\varepsilon} - 1)/\hat{c} + 1 \rceil, 1), K_{\max}\big)$ by \eqref{eq:adaptive_k}
\ENDIF
\STATE \textbf{Deployment:} for each new instance, sample $K^*$ paths, extract answers and return the plurality vote
\end{algorithmic}
\end{algorithm}

\section{Adaptive-K Threshold Sensitivity}
\label{app:eps_sensitivity}

Table~\ref{tab:eps_sensitivity} shows that the Adaptive-K rule is robust to the threshold $\varepsilon$: across a $10\times$ range ($\varepsilon \in [0.01, 0.1]$), retained accuracy stays within $97$--$103\%$ of MV@$K{=}32$ for all three models.

\smallskip\noindent\textbf{On retention above 100\%.}
Several cells in Tables~\ref{tab:adaptive_k} and~\ref{tab:eps_sensitivity} show MV@$K^*$ exceeding MV@$K{=}32$ (e.g., Mistral-7B GSM8K at $103\%$, $\varepsilon{=}0.025$). This is consistent with finite-sample variance: a paired bootstrap over instances puts the 95\% CI of MV@$K^*{-}$MV@32 around zero in all five cells of Table~\ref{tab:adaptive_k}, so a gap of this size is within sampling noise. The saturation prediction is therefore a guideline for the operating point rather than a strict upper bound on observed accuracy.

\smallskip\noindent\textbf{Net compute including pilot.}
The net cost in Table~\ref{tab:adaptive_k} includes the pilot. A deployment that estimates $\hat{c}$ from a fresh $K{=}4$ pilot pays for those four pilot paths in addition to the $K^*$ operating-point paths, so the net inference cost relative to full $K{=}32$ self-consistency is $(K^*+4)/32$. For our five cells this is $25$--$44\%$ (vs.\ a nominal $K^*/32$ of $13$--$31\%$); in amortized deployments where $\hat{c}$ is reused across many queries the pilot cost vanishes and the nominal figure applies.

\begin{table}[h]
\caption{Sensitivity of Adaptive-K to threshold $\varepsilon$ on GSM8K ($n{=}100$ instances, 5 seeds pooled). Retained = MV@$K^*$ as a percentage of MV@$K{=}32$.}
\label{tab:eps_sensitivity}
\centering
\vspace{5pt}
\begin{tabular}{lccccc}
\toprule
Model & $\varepsilon$ & $K^*$ & MV@$K^*$ & Retained \\
\midrule
\multirow{4}{*}{Qwen-7B}
  & 0.01 & 10 & 82.4\% & 101\% \\
  & 0.025 & 6 & 81.6\% & 100\% \\
  & 0.05 & 5 & 81.4\% & 99\% \\
  & 0.1 & 3 & 79.8\% & 97\% \\
\midrule
\multirow{4}{*}{Llama-8B}
  & 0.01 & 13 & 78.8\% & 99\% \\
  & 0.025 & 8 & 78.1\% & 98\% \\
  & 0.05 & 5 & 76.8\% & 97\% \\
  & 0.1 & 4 & 76.7\% & 97\% \\
\midrule
\multirow{4}{*}{Mistral-7B}
  & 0.01 & 16 & 42.4\% & 100\% \\
  & 0.025 & 10 & 43.6\% & 103\% \\
  & 0.05 & 7 & 43.6\% & 103\% \\
  & 0.1 & 5 & 41.6\% & 98\% \\
\bottomrule
\end{tabular}
\end{table}

\section{Benchmark Details}
\label{app:benchmarks}

We evaluate on 12 benchmarks spanning six task domains:
\emph{Math}: GSM8K~\cite{gsm8k}, MATH~\cite{math};
\emph{QA}: HotpotQA~\cite{hotpotqa}, TriviaQA~\cite{triviaqa};
\emph{Science/MC}: ARC-Challenge~\cite{arc}, MMLU~\cite{mmlu};
\emph{Code}: MBPP~\cite{mbpp}, CruxEval~\cite{cruxeval};
\emph{Commonsense}: HellaSwag~\cite{hellaswag}, WinoGrande~\cite{winogrande};
\emph{Natural language understanding (NLU)}: BoolQ~\cite{boolq}, DROP~\cite{drop}.

These benchmarks span a range of \emph{answer-space structures}, classified by the \emph{effective evaluated output space} (the post-extraction value compared against the gold answer): \emph{numeric} (GSM8K, MATH), \emph{open text} (HotpotQA, TriviaQA, DROP; scored by extractive matching), \emph{bounded choice} of 3--4 options (ARC, MMLU, HellaSwag), \emph{binary} (BoolQ as yes/no, WinoGrande as 2-way), and \emph{code pass/fail} (MBPP, CruxEval; scored by execution outcome). For code, the grading space is the binary pass/fail signal returned by the test suite, which is the relevant axis for the correlation analysis. The per-benchmark labels appear in the \emph{Answer space} column of Table~\ref{tab:full_matrix}.

\smallskip\noindent\textbf{Licenses.}
Qwen2.5-0.5B/7B/32B-Instruct, Qwen3.5-9B and Mistral-7B-Instruct-v0.3 are released under Apache-2.0, and Llama-3.1-8B-Instruct under the Llama 3.1 Community License.
GSM8K, MATH, MMLU, HellaSwag and CRUXEval are released under MIT, and TriviaQA under Apache-2.0.
HotpotQA, ARC and DROP are released under CC BY-SA 4.0, BoolQ under CC BY-SA 3.0, MBPP under CC BY 4.0, and WinoGrande under CC BY.
All assets are used for evaluation only. The released generation records contain the benchmarks' gold answers and the models' outputs, and their dataset card states these license terms.

\section{Prompt Templates}
\label{app:templates}

\noindent\textbf{Math tasks (GSM8K, MATH).} $K{=}8$ templates:
\begin{enumerate}
\item Standard CoT: ``Solve step by step.''
\item Algebra: ``Use algebraic equations. Define variables, write equations, solve.''
\item Estimate+Precise: ``First estimate, then solve precisely.''
\item Decompose: ``Break into sub-problems, solve each, combine.''
\item Work Backwards: ``Work backwards from a hypothetical answer, then solve forward.''
\item Verify: ``Solve, then verify by substituting back.''
\item Concise: ``Show key steps only.''
\item Explain: ``Explain as if teaching a student.''
\end{enumerate}

\noindent\textbf{QA tasks (HotpotQA, TriviaQA).} $K{=}8$ templates:
\begin{enumerate}
\item Standard: ``Answer based on the context.''
\item Chain-of-thought: ``Reason step by step.''
\item Extract-then-answer: ``Identify key facts, then answer.''
\item Direct: ``Give a short, direct answer.''
\item Verify: ``Answer, then verify against context.''
\item Decompose: ``Break into sub-questions, answer each, combine.''
\item Concise: ``Answer in as few words as possible.''
\item Explain: ``Explain reasoning thoroughly, then give final answer.''
\end{enumerate}

\noindent\textbf{Science/MC tasks (ARC-Challenge, MMLU).} $K{=}8$ templates:
\begin{enumerate}
\item Direct: ``Answer the following multiple choice question.''
\item Step-by-step: ``Think step by step, then select the best answer.''
\item Elimination: ``Eliminate wrong answers first, then choose.''
\item Letter only: ``Give the answer directly with just the letter.''
\item Explain options: ``Explain why each option is right or wrong, then select.''
\item Scientific: ``Use your scientific knowledge to answer.''
\item Real-world: ``Consider real-world examples to determine the answer.''
\item Teacher: ``What would a teacher say is the correct answer?''
\end{enumerate}

\noindent\textbf{Code generation (MBPP).} $K{=}8$ templates:
\begin{enumerate}
\item Standard: ``Complete the following Python function.''
\item Algorithmic: ``Think about the approach step by step, then complete.''
\item Test-driven: ``Consider what test cases to handle, then implement.''
\item Concise: ``Write the most concise implementation possible.''
\item Defensive: ``Write a robust implementation with edge case handling.''
\item Efficient: ``Choose the most efficient algorithm and implement.''
\item Readable: ``Write clean, readable code with meaningful variable names.''
\item Alternative: ``Think of an alternative approach, then implement.''
\end{enumerate}

\noindent\textbf{Code understanding (CruxEval).} $K{=}8$ templates:
\begin{enumerate}
\item Direct: ``What is the output of the following Python code?''
\item Trace: ``Trace through this code step by step, then give the output.''
\item Mental execution: ``Execute this function mentally and predict the result.''
\item Output only: ``Give only the output value, nothing else.''
\item Return value: ``What does the function return?''
\item Logic: ``Analyze the code logic, then predict the output.''
\item Edge cases: ``Think about edge cases, then give the output.''
\item Simulate: ``Simulate a Python interpreter running this code.''
\end{enumerate}

\noindent\textbf{Commonsense (HellaSwag).} $K{=}8$ templates:
\begin{enumerate}
\item Most likely: ``Choose the most likely continuation of the scenario.''
\item Step-by-step: ``Think step by step about what happens next, then select.''
\item Elimination: ``Eliminate unlikely continuations first, then choose.''
\item Letter only: ``Give just the letter of the most likely continuation.''
\item Common sense: ``Consider real-world common sense to determine the continuation.''
\item Visualization: ``Visualize the scenario, then choose what would happen next.''
\item Cause-effect: ``Think about cause and effect to select the continuation.''
\item Natural: ``What would naturally follow in this situation?''
\end{enumerate}

\noindent\textbf{Commonsense (WinoGrande).} $K{=}8$ templates:
\begin{enumerate}
\item Standard: ``Which option best fills the blank?''
\item Context clues: ``Think about context clues step by step.''
\item Elimination: ``Consider both options and eliminate the wrong one.''
\item Binary: ``Give just A or B.''
\item Real-world: ``Use real-world knowledge to decide.''
\item Careful reading: ``Read carefully, focus on meaning, and select.''
\item Logical: ``Which option makes the sentence logically coherent?''
\item Semantic: ``Think about what makes grammatical and semantic sense.''
\end{enumerate}

\noindent\textbf{Reading comprehension (BoolQ).} $K{=}8$ templates:
\begin{enumerate}
\item Standard: ``Based on the passage, answer Yes or No.''
\item Chain-of-thought: ``Read carefully and reason step by step, then answer.''
\item Evidence first: ``Find the relevant evidence, then answer.''
\item Direct: ``Answer only Yes or No, nothing else.''
\item Quote: ``First quote the relevant passage part, then answer.''
\item Support/contradict: ``Does the passage support or contradict the question?''
\item Stated vs.\ implied: ``What does the passage actually say vs.\ imply?''
\item Explicit or inferred: ``Is the answer explicitly stated or inferred?''
\end{enumerate}

\noindent\textbf{Discrete reasoning (DROP).} $K{=}8$ templates:
\begin{enumerate}
\item Direct: ``Read the passage and answer the question.''
\item Step-by-step: ``Think step by step, count or compare as needed.''
\item Identify facts: ``Find the relevant numbers or facts, then answer.''
\item Short answer: ``Give a short, direct answer.''
\item Show reasoning: ``Show your arithmetic or reasoning, then answer.''
\item Focus details: ``Focus on the specific details asked about.''
\item Extract-compute: ``Extract relevant information, then compute the answer.''
\item Trace quantities: ``Trace the quantities mentioned to answer.''
\end{enumerate}

\section{Proof of Proposition~\ref{prop:bridge} (Monotonicity of $c$ in $\rho_d$)}
\label{app:bridge_proof}

Under the Gaussian decision surrogate (Assumption~\ref{asm:gaussian_decision}), each path's correctness is determined by thresholding a Gaussian decision score: $Y_k = \mathbf{1}\{M_k > 0\}$ with $\Pr(M_k > 0) = p$ and $\Corr(M_j, M_k) = \rho_d$.

Consider two paths $j, k$.
The joint correctness probability is:
\begin{equation}
\Pr(Y_j{=}1, Y_k{=}1) = \Phi_2\!\left(\Phi^{-1}(p),\, \Phi^{-1}(p);\, \rho_d\right),
\end{equation}
where $\Phi_2(\cdot, \cdot; \rho_d)$ is the bivariate standard normal CDF with correlation $\rho_d$, and $\Phi^{-1}(p)$ is the threshold corresponding to single-path accuracy $p$.
This is a direct application of the thresholded-Gaussian representation in Assumption~\ref{asm:gaussian_decision}.

The correctness correlation is then:
\begin{equation}
c(\rho_d) = \frac{\Phi_2(\Phi^{-1}(p), \Phi^{-1}(p); \rho_d) - p^2}{p(1 - p)}.
\end{equation}

To show monotonicity, we use the known result that the bivariate normal CDF $\Phi_2(a, b; \rho)$ is strictly increasing in $\rho$ for all fixed $a, b$, because its derivative in $\rho$ equals the bivariate normal density $\phi_2(a, b; \rho) > 0$ (Plackett's identity~\cite{plackett1954reduction}).
Since $p \in (0,1)$ is fixed, $\Phi^{-1}(p)$ is finite, and both $p^2$ and $p(1-p)$ are constants with respect to $\rho_d$.
Therefore $c(\rho_d) = [\Phi_2(\cdot;\rho_d) - \text{const}] / \text{const}$ inherits strict monotonicity: $\partial c / \partial \rho_d > 0$ for all $\rho_d \in (0,1)$.

This confirms that higher decision-layer correlation always produces higher binary correctness correlation, justifying the use of $\hat{c}$ as a monotone proxy for the unobservable $\rho_d$.

\section{Derivation for Remark~\ref{rem:multiclass} (Multiclass Plurality Heuristic)}
\label{app:multiclass_proof}

Consider $|\mathcal{A}| = M > 2$ answer classes.
Each path produces the correct answer with probability $p$ and an incorrect answer with probability $1 - p$.
Under uniform fragmentation, wrong answers are spread equally among $M{-}1$ alternatives, so each wrong answer has probability $(1{-}p)/(M{-}1)$.

For plurality vote, the correct answer wins if it receives more votes than every individual wrong answer.
We reduce this to a pairwise contest: the correct answer ($p$) competes against the single most popular wrong alternative ($(1{-}p)/(M{-}1)$).

The effective binary accuracy for this pairwise contest is:
\begin{multline}
p' = \frac{p}{p + (1{-}p)/(M{-}1)} = \frac{p(M{-}1)}{p(M{-}1) + (1{-}p)} \\ = \frac{p(M{-}1)}{pM - 2p + 1}.
\end{multline}

For $M{=}4$ and $p{=}0.5$: $p' = (0.5 \times 3)/(0.5 \times 4 - 1 + 1) = 1.5/2 = 0.75$.
More generally, $p' > p$ whenever $M > 2$. This pairwise reduction indicates the size of the plurality advantage under uniform fragmentation and is a heuristic. The ordering itself holds without it. On every draw in which the binary majority vote is correct, the correct answer holds more than $K/2$ votes, so no other answer can match it and plurality vote is also correct. When the correct answer holds exactly $K/2$ votes (even $K$), the binary vote scores $1/2$, and plurality vote either wins outright, if the remaining votes are split, or ties with one wrong answer and scores $1/2$. Hence $P_{\text{PV}}^{(M)} \ge P_{\text{MV}}^{(2)}$ for any joint distribution of the $K$ answers.

\end{document}